\documentclass[letterpaper,twocolumn]{article}
\usepackage[margin=1.5cm]{geometry}
\usepackage{times}

\title{Stochastic Bandit Models for Delayed Conversions} 

\author{%
	 {\bf Claire Vernade }\textsuperscript{1,2}, {\bf Olivier Capp\'e}\textsuperscript{1,3}, {\bf Vianney Perchet}\textsuperscript{1,4,5} \\[1ex]
	1 Universit\'e Paris-Saclay \\
	2 LTCI, Telecom ParisTech, \\
	3 LIMSI, CNRS \\
	4 CMLA, ENS Paris-Saclay \\
	5 Criteo Research 
}

\date{}

\usepackage{graphicx} 
\usepackage{subfig} 

\usepackage[numbers]{natbib}

\usepackage{algorithm}
\usepackage{algorithmic}

\usepackage{amsmath, moreverb}

\usepackage{color,soul}
\usepackage{amsfonts,fancybox}
\usepackage{amssymb}
\usepackage{dsfont}
\usepackage{amsthm}
\usepackage{hyperref}
\usepackage{cleveref}
\newcommand{\BlackBox}{\rule{1.5ex}{1.5ex}}
\renewenvironment{proof}{\par\noindent{\bfseries\upshape
		Proof\ }}{\hfill\BlackBox\\[2mm]}

\newtheorem{theorem}{Theorem}
\newtheorem{lemma}[theorem]{Lemma}
\newtheorem{proposition}[theorem]{Proposition}
\newtheorem{remark}[theorem]{Remark}
\newtheorem{corollary}[theorem]{Corollary}

\DeclareMathOperator*{\argmax}{arg\,max}

\usepackage[usenames,dvipsnames]{xcolor}

\usepackage[backgroundcolor=White,textwidth=1in,disable]{todonotes}

\newcommand{\todoo}[2][]{\todo[color=Yellow!15,size=\tiny,#1]{O: #2}} 
\newcommand{\todov}[2][]{\todo[color=LimeGreen!15,size=\tiny,#1]{V: #2}} 

\newcommand{\delayedUCB}{\textsc DelayedUCB}
\newcommand{\delayedKLUCB}{\textsc DelayedKLUCB}
\newcommand{\dP}{d_{\mathrm{Pois}}}
\begin{document}

\maketitle

\begin{abstract} 
  Online advertising and product recommendation are important domains of applications for multi-armed bandit methods. In these fields, the reward that is immediately available is most often only a proxy for the actual outcome of interest, which we refer to as a \emph{conversion}. For instance, in web advertising, clicks can be observed within a few seconds after an ad display but the corresponding sale --if any-- will take hours, if not days to happen. This paper proposes and investigates a new stochastic multi-armed bandit model in the framework proposed by Chapelle (2014) --based on empirical studies in the field of web advertising-- in which each action may trigger a future reward that will then happen with a stochastic delay. We assume that the probability of conversion associated with each action is unknown while the distribution of the conversion delay is known, distinguishing between the (idealized) case where the conversion events may be observed whatever their delay and the more realistic setting in which late conversions are censored.
We provide performance lower bounds as well as two simple but efficient algorithms based on the UCB and KLUCB frameworks. The latter algorithm, which is preferable when conversion rates are low, is based on a Poissonization argument,  of independent interest in other  settings where aggregation of Bernoulli observations with different success probabilities is required.
\end{abstract} 








\todoo{TBD}
\section{INTRODUCTION}
\label{sec:introduction}

Characterizing the relationship between marketing actions and users' decisions is
of prime importance in advertising, product recommendation and customer
relationship management. In online advertising a key aspect of the problem is
that whereas marketing actions can be taken very fast --typically in less than
a tenth of a second, if we think of an ad display--, user's buying decisions
will happen at a much slower rate
\cite{ji2016probabilistic,chapelle2014modeling,rosales2012post}. In the
following, we refer to a user's decision of interest under the generic term of
\emph{conversion}. Chapelle, in \cite{chapelle2014modeling}, while analyzing data
from the real-time bidding company Criteo, observed that, on average, only 35\%
of the conversions occurred within the first hour. Furthermore, about 13\% of the
conversions could be attributed to ad display that were more than two weeks
old. Another interesting observation from this work is the fact that the delay
distribution could be reasonably well fitted by an exponential distribution,
particularly when conditioning on context variables that are available to the
advertiser.

The present work addresses the problem of sequentially learning to select
relevant items in the context where the feedback happens with long delays. By
long we mean in particular that the feedback associated with a fraction of the
actions taken by the learner will be practically unobserved because they will
happen with an excessive delay. In the example cited above, if we were to run
an online algorithm during two weeks, at least 13\% of the actions would not
receive an observable feedback because of delays. A related situation occurs if
the online algorithm is run during, say, one month, but its memory is limited
to a sliding window of two weeks. In Section~\ref{sec:model} below we introduce
models suitable for addressing these two related situations in the framework of
multi-armed bandits. 

Delayed feedback is a topic that has been considered before in the
reinforcement learning literature and we defer the precise comparison between existing approaches
and the proposed framework to Section~\ref{sec:related}. In a nutshell however,
the distinctive features of our approach is to consider potentially infinite
stochastic delays, resulting in some feedback being \emph{censored} (ie. not
observable). Existing works on delayed bandits focus on cases where the
feedback is observed after some delay,  typically assumed to be
finite. In contrast, we assume that delays are random with a distribution
that may have an unbounded support -- although we  require that it
has finite expectation. As a result, some conversion events cannot be
observed within any finite horizon and the proposed learning algorithm must
take this uncertainty into account.



In Section~\ref{sec:model}, we propose discrete-time stochastic multi-armed
bandit models to address the problem of long delays with possibly censored
feedback. We distinguish two situations that correspond to the cases mentioned
informally above: In the \emph{uncensored} model, conversions can be assumed to
be eventually observed after some possibly arbitrarily long delay; In the
\emph{censored} model, it is assumed that the environment imposes that the
conversions related to actions cannot be observed anymore after a finite window
of $m$ time steps.

Assuming that the delay distribution is known, we prove problem-dependent
lower bounds on the regret of any uniformly efficient bandit algorithm for the
censored and uncensored models in Section~\ref{sec:lower-bound}.

In Section~\ref{sec:indices}, we describe efficient anytime policies relying on
optimistic indices, based on the UCB \cite{auer2002finite} or
KLUCB \cite{garivier2011klucb} algorithms. The latter uses a Poissonization
argument that can be of independent interest in other bandit models. In typical
scenarios where the conversion rates are less than one percent, using the KLUCB
variant will ensure a much faster learning and provides near-optimal perfomance on the long run (see Theorem~\ref{th:klucb}).

These algorithms are analyzed in Section~\ref{sec:algorithms}, showing that
they reach close to optimal asymptotic performance. In
Section~\ref{sec:experiments} we discuss the implementation of these
methods, showing that it is further simplified in the case of geometrically
distributed delays, and we illustrate their performance on simulated
data.


\section{A STOCHASTIC MODEL FOR THE DELAYS}
\label{sec:model}

We now describe our bandit setting for delayed conversion events, inspired by \cite{chapelle2014modeling}. We first consider the setting in which delays may be potentially unbounded and then consider the case where censoring occurs.

\subsection{GENERAL BANDIT MODEL UNDER DELAYED FEEDBACK}

At each round $t \in \mathds{N}^*$, the learner chooses an arm $A_t \in \{1,\dots,K\}$. This action simultaneously triggers two independent random variables: \vspace{-.3cm}
\begin{itemize}
	\item $C_t \in \{0,1\}$, is the \emph{conversion indicator} that is equal to 1 only if the action $A_t$ will  lead to a conversion;\vspace{-.3cm}
	\item $D_t \in \mathds{N}$, is \emph{the delay} indicating the number of time steps needed before the conversion -- if any -- be disclosed to the learner. \vspace{-.3cm}
\end{itemize}

 At each round $t$, the agent then receives an integer-valued reward $Y_t$ which corresponds to the number of observed conversions at time $t$:
 \[
 Y_t = \sum_{s=1}^t C_s \mathds{1}\{D_s = t-s \} . 
 \]
 In the following, we will use the short-hand notation $X_{s,t} = C_s \mathds{1}\{D_s \leq  t-s \}$, for $s \leq t$ to denote the possible contribution of the action taken at time $s$ to the conversion(s) observed at a later time $t$.
We emphasize that even if the actual reward of the learner is the sum of the conversions, we assume that the agent also observes all the individual contributions $(X_{s,t})_{1 \leq s \leq t}$ at time $t$ triggered by actions taken before time $t$.

The above mechanism implies that if $C_t = 1$, the learner will observe $D_t$ at time $t+D_t$, whereas if $C_t = 0$, the delay will not be directly observable. In particular, if at time $t$, $X_{s,u} = 0$, for all $s \leq u \leq t$, it is impossible to decide whether $C_s = 0$ or if $C_s = 1$ but $D_s > t-s$. 
Formally, the history of the algorithm is the $\sigma$-field generated by $\mathcal{H}_t:=(X_{s,u})_{1 \leq u \leq t, 1 \leq s \leq u}$.
\todov{Changed it to gain some lines}

We consider the stochastic model under the following basic assumptions:
\begin{align*}
& C_t | \mathcal{H}_{t-1} \sim \operatorname{Bernoulli}(\theta_{A_t}) ,\\
& D_t | \mathcal{H}_{t-1} \sim \text{distribution  with CDF $\tau$},
\end{align*}
and $C_t,D_t$ are conditionally independent given $\mathcal{H}_{t-1}$.

\begin{lemma}
\label{lemma:regret}
 Denote by $a^* \in \{1,\dots,K\}$ an index such that $\theta_a^* \geq \theta_k$, for $k\in \{1,\dots,K\}$, and define by $r(T) = \sum_{t=1}^T Y_t$ the cumulated reward of the learner and by $r^*(T)$ the cumulated reward obtained by an oracle playing $A_t = a^*$ at each round. The expected regret of the learner is given by
  \begin{equation}
    \label{eq:regret}
    L(T) = \mathds{E}\left[r^*(T)-r(T)\right] = \sum_{s=1}^T \mathds{E}\left[\theta_{a^*}-\theta_{A_s}\right] \tau_{T-s} 
  \end{equation}
  where by definition $\tau_{T-s} = \mathds{P}(D_s \leq T-s)$.
  Denoting  $\mathds{E}[N_k(T)]:= \sum_{s=1}^{T-1} \mathds{1}\{A_s=k\}$, it holds that 
  $$
  L (T) \leq \sum_{k=1}^K (\theta_{a^\star}-\theta_k) N_k(t)
  $$
  and if $\mu = \mathds{E}[D_s]<\infty$,
  \begin{equation}
        \label{eq:regret:relationship}
        \sum_{k=1}^K (\theta_{a^\star}-\theta_k) N_k(t) - L (T) \leq  \mu \sum_{k=1}^K (\theta_{a^\star}- \theta_k).
  \end{equation}
\end{lemma}

\begin{proof}
The cumulated reward at time $T$ satisfies
\begin{multline*}
	r(T) = \sum_{t=1}^T Y_t = \sum_{t=1}^T \sum_{s=1}^t C_s \boldsymbol{1}\{D_s = t-s \} \nonumber \\
	= \sum_{s=1}^T C_s  \boldsymbol{1}\{D_s \leq T-s \} = \sum_{s=1}^T X_{s,T},
\end{multline*}
where the index $T$ stands for the time at which all past conversions are observed while $s$ is the index at which the action has been taken. Hence Eq.~\eqref{eq:regret} is obtained by
\begin{align*}
	 \mathds{E} \left[ \sum_{t=1}^T r(T) \right] &= \mathds{E} \left[ \sum_{t=1}^T Y_t \right]\\
	 & = \sum_{s=1}^T  \mathds{E} \left[ X_{s,T}\right] =\sum_{s=1}^{T} \theta_{A_s} \tau_{T-s}.
\end{align*}

Obviously the fact that $\tau_{T-s} \leq 1$ implies that $L(T)$ is upper bounded by $\sum_{k=1}^K (\theta_{a^\star}-\theta_k) N_k(t)$, which corresponds to the usual regret formula in the bandit model with explicit immediate feedback. To upper bound the difference, note that
\begin{align*}
  \sum_{k=1}^K & (\theta_{a^\star}-\theta_k) N_k(t) - L(T) \\
   & = \sum_{k=1}^K (\theta_{a^\star}-\theta_k) \sum_{s=1}^T \boldsymbol{1}\{A_s=k\} (1-\tau_{T-s}) \\
   & \leq \sum_{k=1}^K (\theta_{a^\star}-\theta_k) \sum_{n=0}^{\infty} (1-\tau_{n}) 
   = \mu \sum_{k=1}^K (\theta_{a^\star}-\theta_k).  
\end{align*}
\end{proof}

\subsection{THRESHOLDED DELAYS: CENSORED OBSERVATIONS}

The model with $m$-thresholded delays takes into account the fact that a conversion can only be observed within $m$ steps after the action occurred. This basically changes the expression of the expected instantaneous reward $Y_t$ which becomes,
	\[
	Y_t = \sum_{s=t-m}^t C_s \boldsymbol{1}\{D_s = t-s \}
	\]
and the future contributions of each action are capped to the next $m$ time steps: $(X_{s,t})_{t-m \leq s \leq t}$.
The history of the algorithm only consists of $\mathcal{H}_t = \sigma((X_{s,u})_{1 \leq u \leq t, u-m \leq s \leq u})$
and the regret expression of Lemma~\ref{lemma:regret} can be split into two terms corresponding to old pulls and the $m$ most recent pulls:
\begin{equation}
\label{eq:regret_th}
\sum_{s=1}^{T-m} (\theta_{a^\star}-\mathds{E}[\theta_{A_s}]) \tau_{m}  + \sum_{s=T-m+1}^T (\theta_{a^\star}-\mathds{E}[\theta_{A_s}]) \tau_{T-s} 
\end{equation}

In the remaining, for $(p,q)\in [0,1]^2$, we will denote by $d(p,q)=p\log(p/q)+(1-p)\log((1-p)/(1-q))$ 
the binary entropy between $p$ and $q$, that is the Kullback-Leibler divergence between Bernoulli distributions with parameters $p$ and $q$. Moreover, without loss of generality, we will assume that $a^*=1$ is the unique optimal arm of the considered bandit problems and denote by $\theta^* = \theta_1$ the optimal conversion rate.

\todov{Why did we remove the following (I found it interesting):

" It is possible to recast censored observations as un-censored observations by modifying the probability of conversion $\theta_k$ and the CDF, namely by defining $\theta'_k=\theta_k\tau_m$ and $\tau's=\tau_s/\tau_m$. However, we are able to derive more precise and powerful results with censored than un-censored data thus we prefer to keep and describe both models. "}

\section{RELATED WORK ON DELAYED BANDITS} 
\label{sec:related}

Delayed feedback recently received increasing attention in the bandit and online learning literature due to its various applications ranging from online advertising \cite{chapelle2014modeling} to distributed optimization \cite{jun2016top,cesa2016delay}.
Indeed, delayed feedback have been extensively considered in the context of Markov Decision Processes (MDPs) \cite{katsikopoulos2003markov,walsh2009learning}. However, the present work focuses on unbounded delays and the models considered therein would result in an infinite space MDP for which even the planning problem would be challenging. In contrast, the lack of memory in bandits makes it possible to propose relatively simple algorithms even in the case where the delays may be very long.
For a review of previous works in online learning in the stochastic and non-stochastic settings, see \cite{joulani2013onlinelong} and references therein. The latter work tackles the 
more general problem of partial monitoring under delayed feedback, with Sections 3.2 and 4 of the paper focusing on the stochastic delayed bandit problem. A key insight from this work is that, in minimax analysis, delay increases the regret in a multiplicative
way in adversarial problems, and in an additive way in stochastic problems.

The algorithm of \cite{joulani2013online} relies on a queuing principle termed
\textsc{Q-PMD} that uses an optimistic bandit referred to as ``BASE'' to
perform exploration; in~\cite{joulani2013online} UCB is chosen as BASE strategy
while the follow-up work \cite{joulani2013onlinelong} also considers the use of
KLUCB.  The idea is to store all the observations that arrive at the same time
$t$ in a FIFO buffer and to feed BASE with the information related to an arm
$k$ only when this arm is about to be chosen. It means that the number of
draws of an arm as well as the cumulated sum of the subsequent rewards are only
updated whenever the observation arrives to the learner. Meanwhile, the
algorithm acts as if nothing happened.

However, in the setting considered in the present work, updating counts only
after the observations are eventually received cannot lead to a practical
algorithm: Whenever a click is received, the associated reward is 1 by
definition, otherwise the ambiguity between non-received and negative feedback
remains.  Thus, the empirical average of the rewards for each arm computed by
the updating mechanism of \textsc{Q-PMD} sticks to 1 and does not allow to
compare the arms. As a consequence, the \textsc{Q-PMD} policy cannot be used
for the models described in Section~\ref{sec:model}, except in the specific
case of the uncensored delay model with bounded delays: Then there is no
censoring anymore as one only needs to wait long enough (longer than the
maximal possible delay) to reveal with certainty the exact value of the
feedback.

Also, \cite{mandel2015queue} notices that the empirical performances of this
queuing-based heuristic are not fully satisfying because of the lack of
variability in the decisions made by the policy while waiting for
feedback. Their suggestion is to use random policies instead of deterministic
ones in order to improve the overall exploration.  Note that even though we
stick to deterministic, history-based, policies, this problem is taken care of
by our algorithm thanks to the use of the CDF of the delays that allow us to
correct the confidence intervals continuously after a pull has been made.

Another possible way to handle bounded delays would be to plan ahead the
sequence of pulls by batches, following the principles of Explore Then Commit,
see \cite{PerchetAl2016}. With finite delays, a new un-necessary batch of
exploration pulls might be started before the algorithm enters the exploitation
(or commitment) phase. The extra cost would therefore be the maximal observable
delay. Although these techniques are random and not deterministic, they have
the same drawbacks as the other ones: The policy is not updated while waiting
for feedback and, as a consequence, cannot handle arbitrarily large delays.

An obvious limitation of our work is that we assume that the delay distribution
is known. We believe that it is a realistic assumption however as the delay
distribution can be identified from historical data as reported
in~\cite{chapelle2014modeling}. In addition, as we assume that the same delay
distribution is shared by all actions, it is natural to expect that estimating
the delay distribution on-line can be done at no additional cost in terms of
performance. Perhaps more interestingly, it is possible to extend the model so
as to include cases where the context of each action is available to the
learner and determines the distribution of the corresponding delay, using for
instance the generalized linear modeling of~\cite{chapelle2014modeling}. In
particular, the same algorithms can be used in this case, by considering the
proper CDFs corresponding to different instances. Of
course the analysis to be described below would need to be extended to cover
also this contextual case.


\section{LOWER BOUND ON THE REGRET}
\label{sec:lower-bound}
The purpose of this section is to provide lower bounds on the regret of \emph{uniformly efficient} 
algorithms in the two different settings of the Stochastic Delayed Bandit problem that we consider.
This class of policies, introduced by \cite{lai1985asymptotically}, refers to algorithms such that for any
bandit model $\nu$, and any $\alpha \in (0,1)$, $\mathds{E}[R(T)]/T^\alpha \to 0$ when $T\to \infty$.

Our results rely on changes of measure argument that are encapsulated in Lemma~1 of \cite{kaufmann2015complexity}, 
or more recently, and more generally, in Inequality (F) of \cite{garivier2017explore}. Those results can actually be reformulated as a lower bound
on the expected log-likelihood ratio of the observations under the originally considered bandit model $\theta$ and the alternative one $\theta'$
\[
\mathds{E}[\ell_T]=\mathds{E}_\theta\left[ \frac{p_\theta((X_{s,t})_{1 \leq t \leq T, 1 \leq s \leq t})}{p_{\theta'}((X_{s,t})_{1 \leq t \leq T, 1 \leq s \leq t})}\right] .
\]

The following inequality is obtained using proof techniques from Appendix B of \cite{kaufmann2015complexity} that are detailed in Appendix~\ref{ap:lower_bound}.

\begin{equation}
\label{eq:llr_bound}
\liminf_{T\to \infty} \frac{\mathds{E}[\ell_T]}{\log(T)} \geq 1	.
\end{equation}

To obtain explicit regret lower bounds for the models introduced in Section~\ref{sec:model}, we compute below the expected log-likelihood ratio corresponding to these two models.

\begin{lemma}
\label{lem:llr-uncens}
	In the censored delayed feedback setting, the expected log-likelihood ratio is given by
	\begin{align*}
	\mathds{E}_\theta\left[ \ell_T\right] =& \sum_{s=1}^{T-m} d(\theta_{A_s}\tau_{m}, \theta_{A_s}'\tau_{m}) \\
	& + \sum_{s=T-m}^T d(\theta_{A_s}\tau_{T-s}, \theta_{A_s}'\tau_{T-s}).	
	\end{align*}
	In the uncensored setting, the above sum is not split and we have
	\[
	\mathds{E}_\theta\left[ \ell_T\right] = \sum_{s=1}^T d(\theta_{A_s}\tau_{T-s}, \theta_{A_s}'\tau_{T-s}).
	\]
\end{lemma}

\begin{proof}
	Given $\mathcal{H}_{s-1}$, $(X_{s,s}, \dots, X_{s,T})$ can be equal to
	\begin{itemize}
		\item $(0, \dots, 0)$, with proba.\ $(1-\theta_{A_s}) + \theta_{A_s} (1-\tau_{T-s})$,
		\item $(0, \dots, 0, 1, 1, \dots, 1)$ with proba.\ $\theta_{A_s} \delta_{u-s}$, for $u=s, \dots, T$ ($u$ denotes the position of 1 in the vector), where $\delta_k = \mathds{P}(D_s \leq k)$. 
	\end{itemize}
	
	Hence,
	\begin{align*}
	& \mathds{E}_\theta\left[ \left. \log \frac{p_\theta(X_{s,s}, \dots, X_{s,T})}{p_{\theta'}(X_{s,s}, \dots, X_{s,T}} \right| \mathcal{H}_{s-1} \right] \\
	& =  \log \frac{1-\theta_{A_s}\tau_{T-s}}{1-\theta_{A_s}'\tau_{T-s}} (1-\theta_{A_s}\tau_{T-s}) \\
        & \qquad \qquad + \sum_{u=s}^T \log \frac{\theta_{A_s} \delta_{u-s}}{\theta_{A_s}' \delta_{u-s}} \theta_{A_s} \delta_{u-s} \\
	& =  \log \frac{1-\theta_{A_s}\tau_{T-s}}{1-\theta_{A_s}'\tau_{T-s}} (1-\theta_{A_s}\tau_{T-s}) + \log \frac{\theta_{A_s}}{\theta_{A_s}'} \theta_{A_s}\tau_{T-s} \\
	& = d(\theta_{A_s}\tau_{T-s}, \theta_{A_s}'\tau_{T-s}).
	\end{align*}
	The equivalent expression for the censored case is easily deduced from the same calculations.
\end{proof}


\subsection{CENSORED SETTING}
Using our notations, the following theorem provides a problem-dependent lower bound on the regret. 
\begin{theorem}
	\label{th:lb-cens}

	The regret of any uniformly efficient algorithm is bounded 
	from below by	
	\[
	\liminf_{T \to \infty} \frac{R(T)}{\log(T)} \geq \sum_{k\neq k^*} 
	\frac{\tau_m(\theta^* - \theta_k)}{d(\tau_m \theta_k, \tau_m \theta^*)} .
	\]
\end{theorem}

\begin{proof}
The details of the proof can be found in Appendix~\ref{ap:lower_bound} but we provide here a sketch of the main argument.
The log-likelihood ratio is given by Lemma~\ref{lem:llr-uncens}:
\begin{align*}
	\mathds{E}_\theta\left[ \ell_T\right] =& \sum_{s=1}^{T-m} d(\theta_{A_s}\tau_{m}, \theta_{A_s}'\tau_{m}) \\
	& + \sum_{s=T-m}^T d(\theta_{A_s}\tau_{T-s}, \theta_{A_s}'\tau_{T-s}) ,
\end{align*}
which is bounded from below by Eq.\eqref{eq:llr_bound}. However, obtaining a lower bound on the regret requires to decompose this quantity into $(K-1)$ terms depending on the suboptimal arms. For a fixed arm $k\neq 1$, we consider $\theta'=(\theta_1,\ldots,\theta_{k-1},\theta_1+\epsilon,\ldots,\theta_K)$ for which the expected log-likelihood ratio is 
\begin{align*}
	&\mathds{E}[N_k(T)] d(\tau_m\theta_k,\tau_m(\theta_1+\epsilon)) \\
	& + \sum_{s=T-m}^T d(\theta_k\tau_{T-s}, (\theta_1+\epsilon)\tau_{T-s} )\geq \mathds{E}_\theta\left[ \ell_T\right].	
\end{align*}

Divide by $\log(T)$ and let $T$  to infinity, to get the result for $\epsilon \to 0$, as the second term in the l.h.s. is bounded.
\end{proof}

This lower bound implies that the delayed bandits problem with trespassing
probability $\tau_m$ is as hard as solving the scaled bandit problem with
expected rewards $(\tau_m\theta_1,\ldots,\tau_m\theta_K)$. In
the long run, one cannot  learn faster than the heuristic approach 
discarding the last $m$ observations and considerimg the fictitious bandit model
with parameters $(\tau_m\theta_1,\ldots,\tau_m\theta_K)$. However, on horizons
of the order of $m$ time-steps, we will show empirically in
Section~\ref{sec:experiments} that taking  delay distributions into account
 allows for much faster learning. Note also that the convexity of the function $\tau \to d(\tau p, \tau q)$ proved in Lemma~\ref{lem:incrkappa} implies that the regret
lower bound is a monotonically increasing function of $\tau_m$. Hence, either
reduced values of $m$ or longer values of the expected delay $\mu$ actually make
the problem harder.

\subsection{UNCENSORED SETTING}

In the uncensored model, the same argument shows that the lower bound does not differ from the classical Lai \& Robbins Lower Bound \cite{lai1985asymptotically}.


\begin{theorem}     
\label{th:lb-uncens}     
The regret of any uniformly
efficient algorithm in the Uncensored Delays Setting is bounded      
from below by        
\[     \liminf_{T \to \infty} \frac{R(T)}{\log(T)} \geq \sum_{k\neq
k^*}      \frac{(\theta^* - \theta_k)}{d(\theta_k,  \theta^*)} .    
\]
\end{theorem}

The full proof of this result is similar to the proof of Theorem~\ref{th:lb-cens} and can be found in Appendix~\ref{ap:lower_bound}.


\section{DELAY-CORRECTED ESTIMATORS AND CONFIDENCE INTERVALS}
\label{sec:indices}

In this section, for a fixed arm $k \in \{1,\dots,K\}$, we define a conditionally unbiased estimator 
for the conversion rate $\theta_k$. Then, based on suitable concentration 
results we derive optimistic indices: a delay-corrected UCB as in~\cite{auer2002finite}
as well as a delay-corrected KLUCB as in \cite{garivier2011klucb}.

\subsection{PARAMETER ESTIMATOR}
Define the sum of rewards up to time $t$ as
$$
S_k(t) = \sum_{s=1}^{t} \sum_{u=1}^{s} X_{u,s}\mathds{1}\{A_u=k\} .
$$
We recall that we  defined the exact number of pulls of arm $k$ up to time $t$ 
as $N_k(t):=\sum_{s=1}^{t-1}\mathds{1}\{A_s = k\}$. However, defining an estimator
of $\theta_k$ that is unbiased -- when conditioning on the selections of arms -- requires  to consider a delay-corrected count $\tilde{N}(t)$  taking into account
the probability of having eventually observed the reward associated with each previous pull of $k$. 
We distinguish the expression of $\tilde{N}(t)$ according to whether  feedback is censored or not.

\paragraph{Censored model.}
When rewards cannot be disclosed after $m$ rounds following the action, the current available information on the pulls is split into  two main groups: 
The `oldest' pulls, censored if not observed yet, and the most recent ones. Namely, we now define $\tilde{N}_k(t)$ as
\[ \tilde{N}_k(t) = \sum_{s=1}^{t-m} \mathds{1}\{A_s=k\}  \tau_{m} \, + \; \sum_{s=t-m+1}^{t-1} \mathds{1}\{A_s=k\}  \tau_{t-s}.
\]

Overall, the conversion rate estimator is defined as
\begin{equation}
\label{eq:estimator}
	\hat{\theta}_k(t) = \frac{S_k(t)}{\tilde{N}_k(t)} .
\end{equation}

\begin{remark}
	In the uncensored case, defining $\tilde{N}_k(t) := \sum_{s=1}^{t} \mathds{1}\{A_s=k\}  \tau_{t-s}$.
leads to an equivalent definition of $\hat{\theta}_k(t)$ as a conditionally unbiased estimator.
\end{remark}


\subsection{UCB INDEX}
We first define a delay-corrected UCB-index for bounded rewards.

\paragraph{Concentration bound.}
Using the self-normalized concentration inequality of Proposition~8 of \cite{lagree2016multiple}, 
yields the following result, that we recall here for completeness. 

\begin{proposition}
	\label{prop:controlUCB}
	Let $k$ be an arm in $\{1,...,K\}$, then for any $\beta >0$ and for all $t>0$, 
	\[
	\mathds{P}\left(\theta_k > \hat{\theta}_k(t) +
	\sqrt{\frac{N_k(t)}{\tilde{N}_k(t)}}
	\sqrt{\frac{\beta}{2\tilde{N}_k(t)}}\right) < \beta e \log(t) 
	e^{-\beta}.
	\]
\end{proposition}

\paragraph{Upper-confidence Bound.} Thus, an UCB index for $\hat{\theta}_k(t)$ may be defined as 
\[
U^{\textsc{ucb}}_k(t) = \hat{\theta}_k(t) +
\sqrt{\frac{N_k(t)}{\tilde{N}_k(t)}}
\sqrt{\frac{\beta_\epsilon (t)}{2\tilde{N}_k(t)}},
\]
where $\beta\epsilon(t)$ is a suitable slowly growing exploration function (see below). This upper confidence interval is  scaled by $N_k(t)/\tilde{N}_k(t)$ when compared to the classical UCB index. 
This ratio gets bigger when the $(\tau_d)$'s are small for large delays $d$, that is when the median delay is large: 
The longer we need to wait for observations to come, the largest our uncertainty about our current cumulated reward.

\subsection{KLUCB INDEX}

\paragraph{Concentration bound.}
We first state a concentration inequality that controls the underestimation
probability based on an alternative Chernoff bound for a sum of independent binary random variables (Lemma \ref{lemma:laplace} proved in  Appendix~\ref{ap:concentration}).

This lemma only holds for a sequence of pulls fixed before-hand, independently of realizations, i.e., the values of $A_t$ do not depend on the sequence of $X_s$. Although with a restrictive scope, it provides intuition on the construction of the algorithm. 

\begin{lemma}
\label{lemma:KL-concentration}
Assume that the sequence of pulls is fixed beforehand and let $k$ be an arm in $\{1,...,K\}$. Then for any $\delta > 0$ and for all $t>0$, 
	\[
	\mathds{P}\left( \left\{\hat{\theta}_k(t) < \theta_k \right\} \cap \left\{\tilde{N}_k(t)\dP(\hat{\theta}_k(t),\theta_k) > \delta\right\} \right) <  
	e^{-\delta}.
	\]
where $\dP(p,q) = p\log p/q + q-p$ denotes the Poisson Kullback-Leibler divergence.	
\end{lemma}
To get upper confidence bounds for $\theta_k$ from Lemma~\ref{lemma:KL-concentration}, we follow~\cite{garivier2011klucb} and define the KL-UCB index by
\begin{multline*}
	U_k^{\textsc{kl}}(t) =  \max \Bigl\{ q \in [\hat{\theta}_k(t), 1] \, : \\
\, \tilde{N}_k(t)\dP(\hat{\theta}_k(t),q)\leq \beta_\epsilon(t) \Bigr\} .
\end{multline*}
Using $\beta_\epsilon(t) = \beta$, this KL-UCB index satisfies a result analogous to Proposition~\ref{prop:controlUCB} (see Proposition \ref{th:self-normalized} in Appendix \ref{ap:Poisson}):
	\[
	\mathds{P}\left( \theta_k > U_k^{\textsc{kl}}(t)\right) \leq e \lceil \beta \log (t) \rceil e^{-\beta} .
	\]

Even though the Kullback-Leibler divergence does not have the same expression for Bernoulli and Poisson random variables, the following lemma (proved in Appendix \ref{ap:Poisson}) shows that for a certain range or parameters they are actually very close.

\begin{lemma} 
	\label{lem:div_bound}
	For $0<p<q<1$,
$$
  (1-q) d(p,q) \leq \dP(p,q) \leq d(p,q).
$$
\end{lemma}

\section{ALGORITHMS}
\label{sec:algorithms}


Algorithm~\ref{algo} present the scheme common to both the censored and uncensored cases, which differ only by the definition of the parameter estimator.
 In both cases, one may also consider either of the two UCB or KL-UCB index defined in the previous section, resulting in the \delayedUCB\ and \delayedKLUCB\ algorithms. We provide a finite-time analysis of the regret of these algorithms, when using an exploration function of the form $\beta_\epsilon(t)=(1+\epsilon)\log(t)$, for some positive $\epsilon$.

\begin{algorithm}[hbt]
	\caption{-- \delayedUCB\ and \delayedKLUCB.}
        \label{algo}
	\begin{algorithmic}\small
		\REQUIRE{$K$, CDF parameters $(\tau_d)_{d\geq 0}$, threshold $m>0$ if feedback is censored.}
		\STATE{Initialization: First $K$ rounds, play each arm once. }
		\FOR{$t>K$}
		\STATE{Compute $S_k(t)$ and $\tilde{N}_k(t)$ for all $k$ according to the assumed feedback model (censored or not),}
		\STATE{Compute $\hat{\theta}_k(t)$ for al $k$,}
		\STATE{For a given choice of algorithm $\mathcal{A} \in \{\textsc{klucb},\textsc{ucb\}}$,}
		\STATE{$A_t \gets \arg\max_k U^{\mathcal{A}}_k(t)$}.
		\STATE{Observe reward $Y_t$ and all individual feedback $(X_s,t)_{s\leq t}$}
		\ENDFOR
	\end{algorithmic}
\end{algorithm}

\paragraph{Finite-time Analysis of \delayedUCB.}

\begin{theorem}\label{th:ucb}

	In the censored setting, the regret of \delayedUCB\ is bounded from above by
	\[
	L_{\textsc{ucb}}(T)\leq (1+\epsilon)\log(T)\sum_{k \neq \ast} \frac{1}{2\tau_m\Delta_k} + o_{\epsilon,m}(\log (T)).
	\]
\end{theorem}

\begin{proof}
Outline of the proof (cf Appendix~\ref{ap:ucb}):
\begin{enumerate}
	\item First upper-bound the regret using Lemma~\ref{lemma:regret} in the uncensored case:
	\[
	R(T) \leq \sum_{k>1} \Delta_k \mathds{E}[N_k(T)],
	\]
	and bounding the first $m$ losses by $1$ in the censored case:
	\[
	R(T) \leq m + \sum_{k>1} \tau_m \Delta_k \mathds{E}\left[ \sum_{t>m}^{T}\mathds{1}\{A_t = k\} \right].
	\]
	\item Then, decompose the event $\mathds{1}\{A_t = k\}$ as in \cite{auer2002finite}
	\begin{multline*}
		\sum_{t>m}^{T}\mathds{1}\{A_t = k\} \leq \sum_{t>m}^{T}\mathds{1}\left\{ U^{\text{UCB}}_1(t) < \theta_1 \right\} \\
		+ \sum_{t>m}^{T}\mathds{1}\left\{ A_{t+1}=k, U^{\textsc{UCB}}_k(t) \geq \theta_1 \right\}.
	\end{multline*}
	\item Remark that the first sum is handled by Proposition~\ref{prop:controlUCB} so it suffices to control the second sum. 
	\begin{align*}
	&\mathds{E}\left[ \mathds{1}\left\{ A_{t+1}=k, U^{\textsc{UCB}}_k(t) \geq \theta_1 \right\} \right]\\
	& \qquad \leq \frac{(1+\epsilon)\log(T)}{2\tau_m^2 \Delta_i^2} \\
	& \qquad \qquad + \sum_{s > \frac{(1+\epsilon)\log(T)}{2\Delta_i^2} } 
	\mathds{P}\left( U^{\textsc{UCB}}_k(t)   \geq \theta_i + \Delta_i\right) .
	\end{align*}

\end{enumerate}
The  last term is  actually $O(\sqrt{\log(T)})$, giving the desired result. Details, 
as well as explicit constants and dependencies can be found in Appendix~\ref{ap:ucb}.
\end{proof}

\begin{corollary}
	\label{cor:uncensUCB}
In the uncensored setting, we also assume that there exists  $c >0$  such that $1-\tau_m \leq \frac{c}{m}$ for all $m \geq 1$. Then, 
is bounded from above by
	\[
	L_{\textsc{ucb}}(T)\leq \frac{1+\epsilon}{1-\epsilon}\log(T)\sum_{k>1} \frac{1}{2\Delta_k} + o_{\epsilon,m}(\log (T)) .
	\]
\end{corollary}
\begin{proof}
The analysis of \delayedUCB\ given in  Appendix~\ref{ap:ucb} (in the censored setting) shows that  the performances of \delayedUCB\ in the uncensored setting can be upper-bounded by its performances in the censored setting, where  the threshold $m$ can be arbitrarily fixed to some value. Choosing $m$ will only have an impact on the analysis of the algorithm. The specific choice of $m$ satisfying  $\tau_m \geq 1-\epsilon$ gives the claimed result.
As indicated  in Appendix~\ref{ap:ucb}, the dependency of $o_{\epsilon,m}(\log (T))$ is actually  only linear in $m$. As a consequence, along with the assumption on the decay of $1-\tau_m$, this yields that the overall dependency in the parameter $m$ is reduced to $1/\epsilon$. 
\end{proof}

We emphasize  that the assumption that  $1-\tau_m \leq 1/m$, is actually rather natural. Indeed, if $1-\tau_m \leq c/m^\gamma$, for some constants $c,\gamma>0$, then the finiteness requirement on the expected delay is satisfied if and only if $\gamma >1$.

\paragraph{Finite-time Analysis of \delayedKLUCB.}
\begin{theorem}\label{th:klucb}

	For any $\eta >0$, the regret of \delayedKLUCB\ is bounded in the censored setting as
	\begin{align*}
			L_{\textsc{klucb}}(T)\leq &(1+\eta)\frac{\beta_\epsilon (t)}{1-\theta_1}\sum_{k > 1} 
			\frac{\tau_m\Delta_k}{d(\tau_m\theta_k, \tau_m\theta_1)} \\
			&+ o_{\epsilon,m,\eta}(\log (T)).
	\end{align*}

\end{theorem}
\vspace{-0.25cm}
\begin{proof}
	Outline of the proof (cf. Appendix\ref{ap:klucb}):
\begin{enumerate}
	\item We start by decomposing the regret according to the different types of unfavorable events. Note that thanks to the upper bound on the regret provided by Lemma~\ref{lemma:regret}, we need to control on the number of suboptimal pulls $\mathds{E}[N_k(T)]$ for arms $k>1$.
	\begin{multline*}
	\mathds{E}[N_k(T)] \leq m + \mathds{E}\left[ \sum_{t=m+1}^T \mathds{1}\{U_1(t)<\theta_1 \} \right] \\
	+ \mathds{E}\left[ \sum_{t=m+1}^T \mathds{1}\{ A(t)=k, U_k(t) \geq  \theta_1 \} \right] .
	\end{multline*}
	
	\item The first sum is handled by Theorem~\ref{th:self-normalized} in Appendix~\ref{ap:concentration} which shows that it is $o(\log(T))$. For the second term, we bound the indices using the fact that $\tilde{N}_k(t) \geq \tau_m N_k(t-m)$ to obtain
	\begin{align*}
	& U^{\textsc{kl}}_k(t) \leq  U^{\textsc{kl}+}_k(t)\\
	&:=
	\argmax_{q \in [\hat{\theta}_k,1]} \bigl\{q|\tau_m  \dP(\hat{\theta}_k(t),q)\leq \frac{\beta_\epsilon(t)}{N_k(t-m)}\bigr\} . 
	\end{align*}
	Notice that the $U^{\textsc{kl}+}_k(t)$  indices are well defined for $t>m$.
	
	\item Then, we proceed as in the proof of Theorem~10 in Appendix B.2 of \cite{joulani2013online}. For any $\eta>0$, we define the characteristic number of pulls
	\[
	K_k(T) = \frac{(1+\eta )\beta_\epsilon (t)}{\dP(\tau_m\theta_k,\tau_m\theta_1)} ,
	\]
	and we prove
	\begin{align*}
			\sum_{s\geq K_k(T)} \mathds{P} \left( \tau_m s \dP(\hat{\theta}_{k,s},\theta_1)
			\leq\beta_\epsilon (t) \right) \\
			= o_{\epsilon,m,\eta} (\log T)
	\end{align*}
	using Fact 2 of \cite{cappe2013kullback} for exponential families.
\end{enumerate}

\end{proof}

\begin{corollary}
	In the uncensored setting, under the same hypothesis than in Corollary~\ref{cor:uncensUCB}, namely that there exists a constant $c$ such that $1-\tau_m\leq \frac{c}{m}$ for all $m\leq 1$. Then, the regret of \delayedKLUCB\ is bounded from above as 
	\begin{align*}
		L_{\textsc{klucb}}(T) & \leq \frac{\beta_\epsilon(t)}{1-\theta_1} \sum_{k > 1} 
		\frac{(1+\eta)(1-\epsilon)\Delta_k}{d((1-\epsilon)\theta_k, (1-\epsilon)\theta_1)} \\
		& + o_{\eta,\epsilon}(\log(T)) .
	\end{align*}

\end{corollary}

\begin{proof}
	As for the proof of Corollary~\ref{cor:uncensUCB}, the performance of \delayedKLUCB\ in the uncensored case can be bounded as in the censored case by a specific choice of $m(\epsilon)$ such that $\tau_m \geq 1-\epsilon$, namely $m(\epsilon)\geq c/\epsilon$. As shown in the proof of Theorem~\ref{th:klucb} in the censored case, the dependency in $m$ of the term of rest is linear, reducing to $1/\epsilon$.
\end{proof}

\paragraph{Naive benchmark: The \textsc{Discarding} policy.}
An obvious benchmark algorithm in the censored setting is to use the regular UCB 
and KLUCB policies only using the first $t-m$ pulls and observed rewards at each time $t$. 
In that case the empirical average considered is simply $\hat{\theta}^m_k(t) = S_k(t-m)/\tau_m N_k(t-m)$ and the
corresponding optimistic indices are
\begin{align*}
& U^{m} (t)= \hat{\theta}^m_k(t) + \sqrt{\beta_\epsilon (t)/2\tau_m N_k(t-m)},\\
& U^{m|\textsc{kl}}_k(t) = 
 \max_{q \in [\hat{\theta}_k^m, 1]} \Bigl\{q |\, \tau_m  \dP(\hat{\theta}_k^m, q) \leq \frac{\beta_\epsilon (t)}{N(t-m)} \Bigr\}.
\end{align*}

These indices can only be computed after at least $m$ rounds. The proof technique used for the analysis of our algorithms in the censored case actually shows that the  \textsc{DiscardingUCB} and \textsc{DiscardingKLUCB} policies are asymptotically optimal. Nonetheless, in practice it
is very undesirable to have an arbitrarily long linear regret phase at the beginning of
the learning until the threshold $m$ is reached. This is  especially true if the threshold $m$ is  large as compared to
the horizon $T$. 
In that case, we empirically show in Section~\ref{sec:experiments} that our algorithms achieve drastically improved short-horizon performance.

\section{EXPERIMENTS}
\label{sec:experiments}

\begin{figure*}[t]
  \centering
  \subfloat[$\theta_H=(0.5,0.4, 0.3)$.]
  {\label{fig:high-rew}
  \includegraphics[width=4cm]{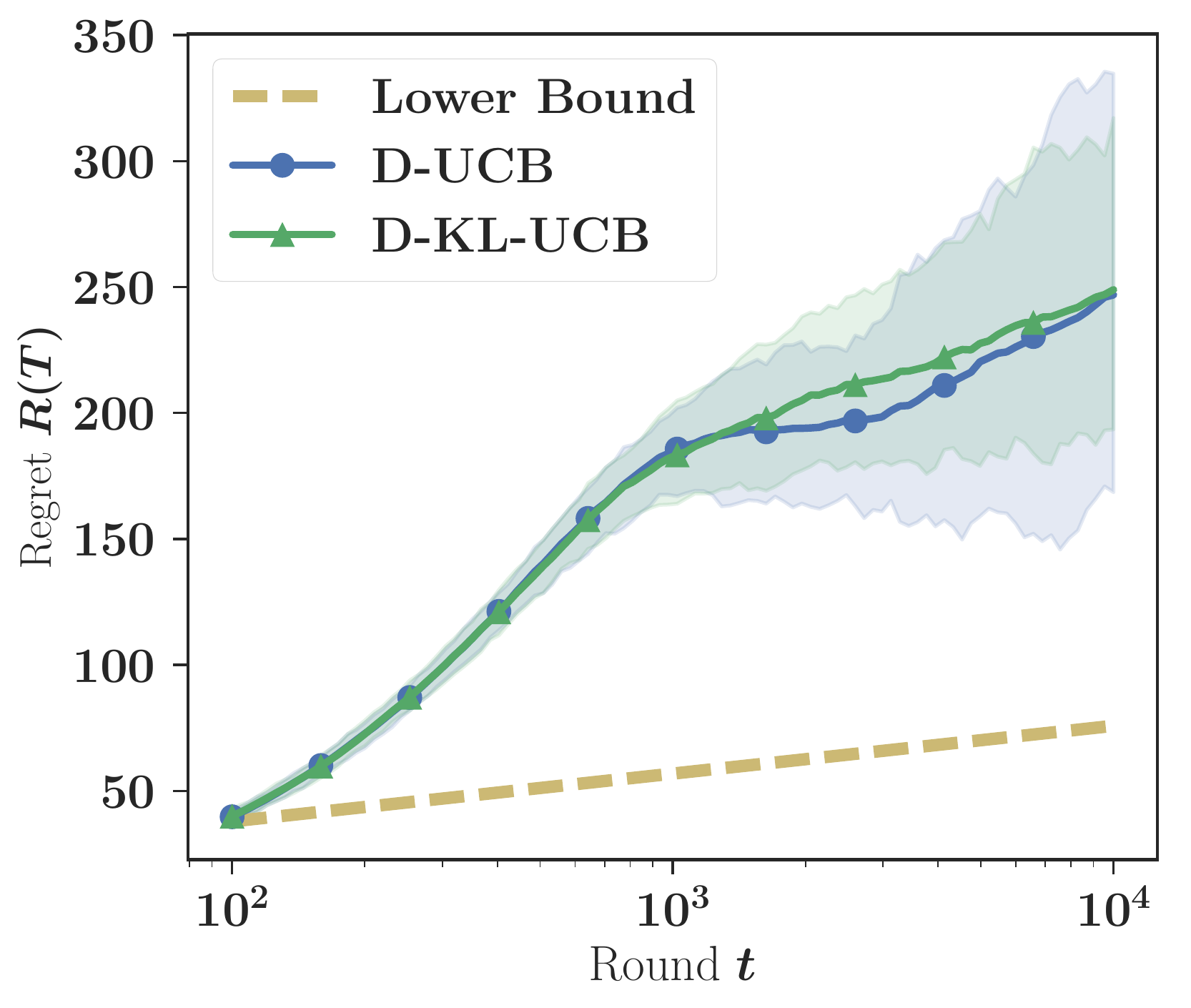}
  \includegraphics[width=4cm]{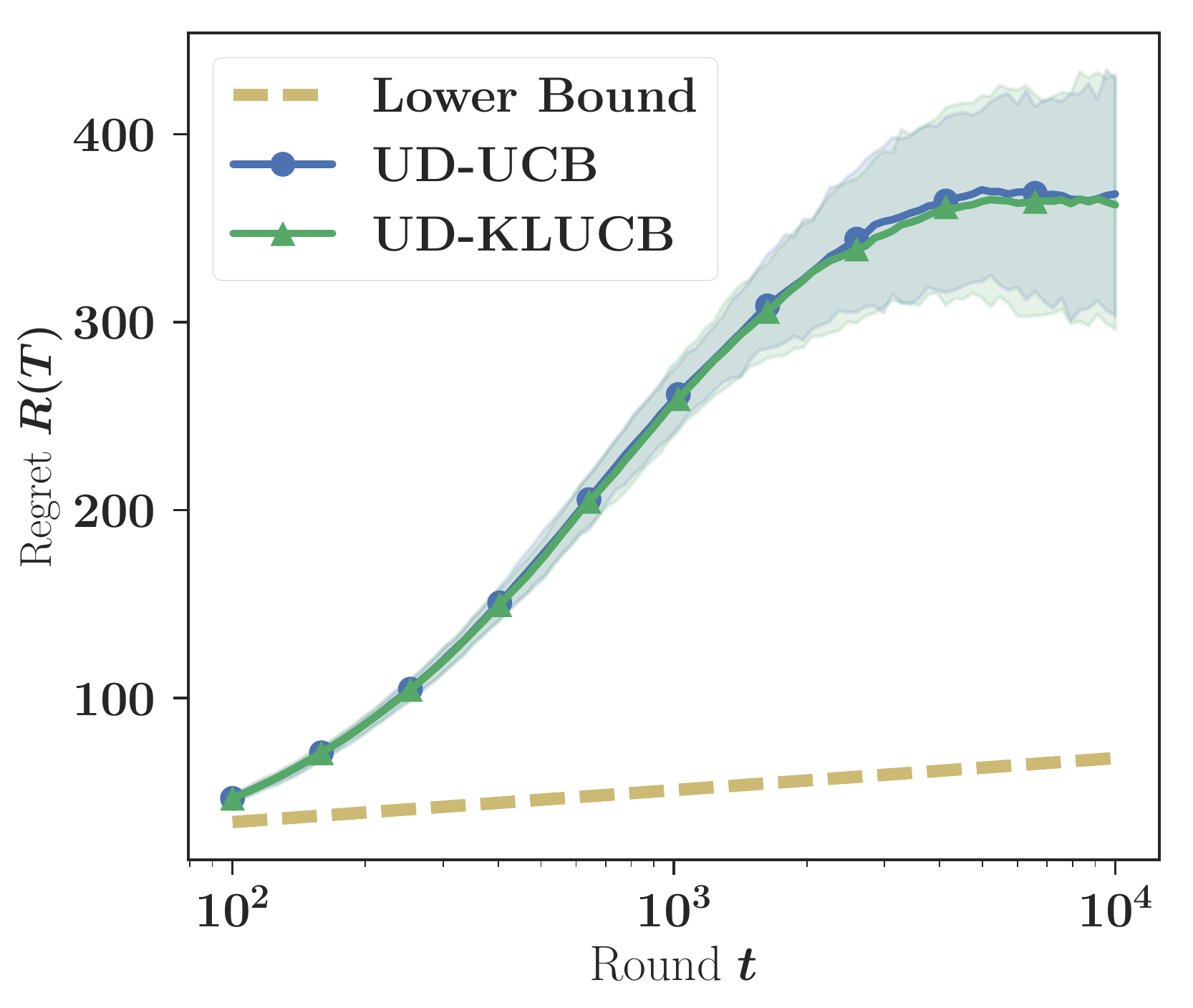}
  }
  ~
  \subfloat[$\theta_L=(0.1,0.05, 0.03)$. ]
  { \label{fig:low-rew}
   \includegraphics[width=4cm]{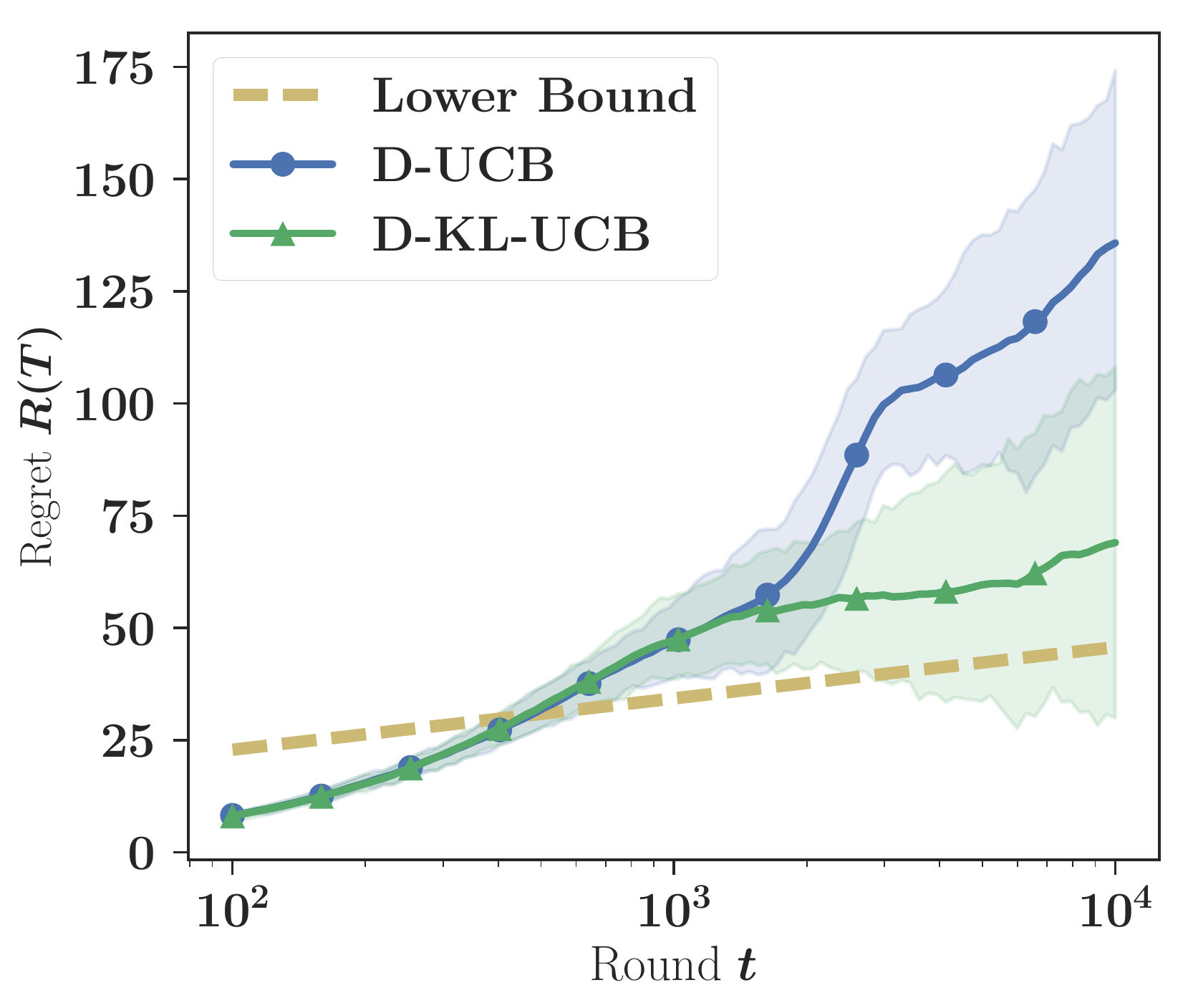}
   \includegraphics[width=4cm]{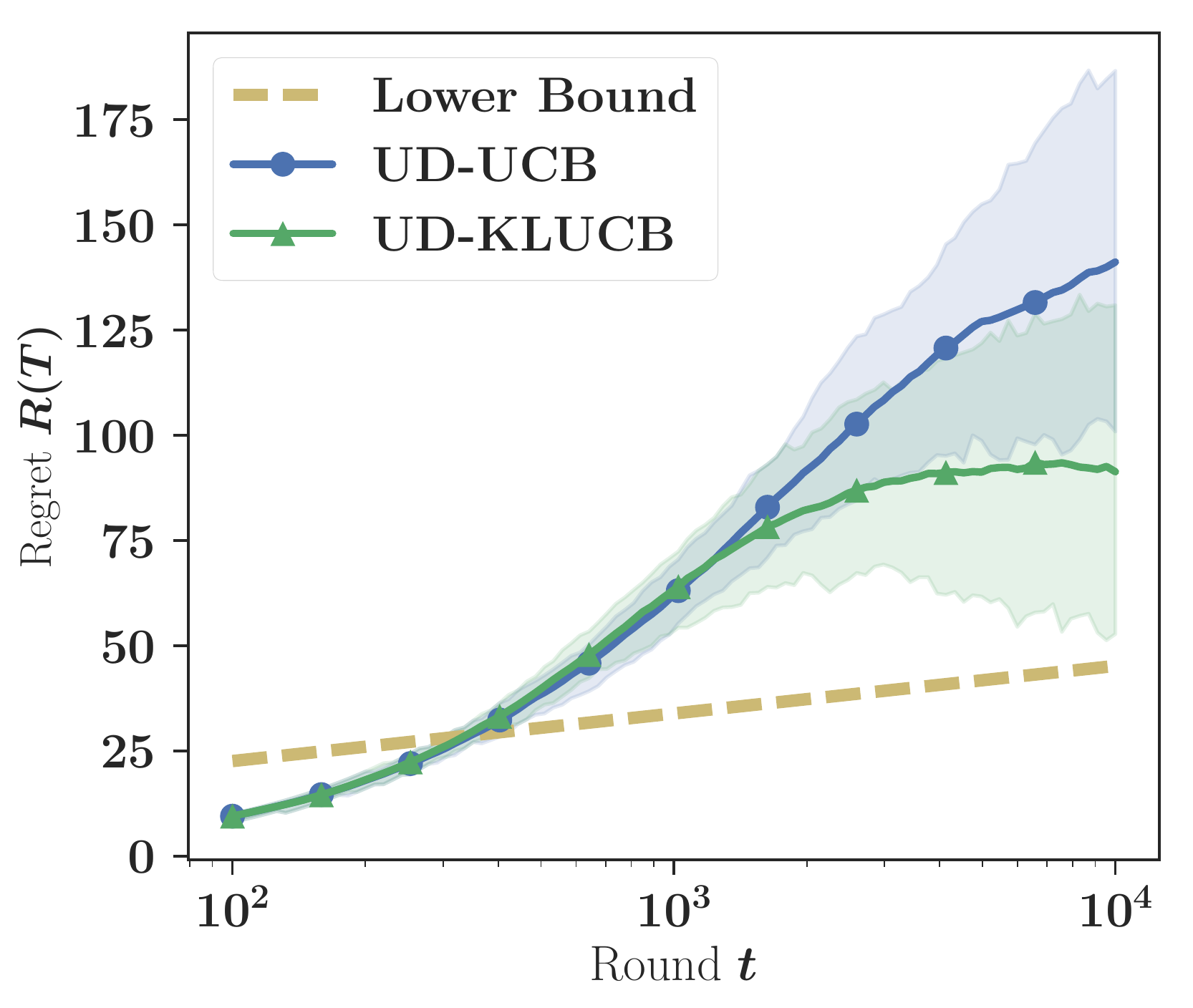} 
   }

 \caption{Expected regret of \textsc{d-ucb} and \textsc{d-klucb} (censored setting), 
 and \textsc{ud-ucb} and \textsc{ud-klucb} (uncensored setting) for two bandit problems: $\theta_H=(0.5,0.4, 0.3)$, $\theta_L=(0.1,0.05, 0.03)$. For all experiments, $T=10000$, $\mu=500$, $m=1000$ (if censored) and the results are averaged over 100 runs.  }
 \label{fig:opt}
\end{figure*}

In this section we perform simulations in our two delayed feedback frameworks. The algorithms described in the previous section will be denoted \textsc{d-ucb} and \textsc{d-klucb} in the censored setting, and \textsc{ud-ucb} and \textsc{ud-klucb} in the uncensored setting.

As a matter of fact, the bottleneck of such policies is to compute $\tilde{N}(t)$ which is theoretically a weighted sum over all past actions and, 
without any assumption on the weights $(\tau_s)_{s\geq 0}$, it requires to store all previous rewards and recompute $\tilde{N}(t)$ at each iteration. 

Following the conclusions of \cite{chapelle2014modeling}, we assume all along this section that the delays follow a geometric distribution with parameter $\lambda := 1/\mu$. 
This assumption allows us to implement our algorithms in a computationally, memory-efficient manner. 
Indeed, for each $s\geq 0$, we now have $(1-\tau_{s+1})=\lambda (1-\tau_s)$ and this remark provides a sequential updating scheme of the quantity $\tilde{N}_k(t)$ for $k\in[K]$.
In the uncensored setting, we have: 
\[
\tilde{N}_k(t) = \sum_{s=1}^t (1-\lambda^{t-s+1})\mathds{1}\{A_s = k\} = N_k(t) - O_k(t) ,
\]
where $O_k(t)$ is updated after each round as follows
\begin{equation}
\label{eq:O}
  O_k(t+1)\gets \lambda O_k(t) + \mathds{1}\{A_t=k\}.
\end{equation}

In the censored setting, however, one must still keep track of some of the previous 
pulls in order to compute 
\[ \tilde{N}_k(t) = N_k(t-m)\tau_{m} +\sum_{s=t-m+1}^{t-1} 
\mathds{1}\{A_s=k\}  \tau_{t-s}.
\]
In practice this can be done by maintaining a buffer of size $m$ 
containing the last $m$ pulls that are multiplied by the 
probability of observing a reward with the delay corresponding to their 
current position in the buffer. In addition to this buffer, we add old 
pulls in a separate count $N_k(t-m)$ for which the weight will stay $\tau_m$. 
	

\paragraph{Comparing \delayedUCB\  and \delayedKLUCB.}

We compare the regret of both delayed bandits policies in the censored and uncensored setting for $T=10000$, $\mu=500$ and $m=1000$.

Simulations on Figure~\ref{fig:opt}, for two problems, $\theta_H=(0.5,0.4, 0.3)$ on the left, and $\theta_L=(0.1,0.05, 0.03)$ on the right,  display the classical pattern that while UCB-based algorithms perform satisfactorily for central values (close to 0.5) of the conversion rate, they are clearly sub-optimal with  more realistic values for the conversion rate. The two right plots also confirm that, for the KLUCB-based algorithms, the loss with respect to the optimal regret growth rate due to the use of the Poisson divergence is -- as expected from Theorem~\ref{th:klucb} -- not significant for low values (here $\theta^*=0.1$) of the conversion rates.


\begin{figure}[hbt]
\centering
\includegraphics[width=0.49\linewidth]{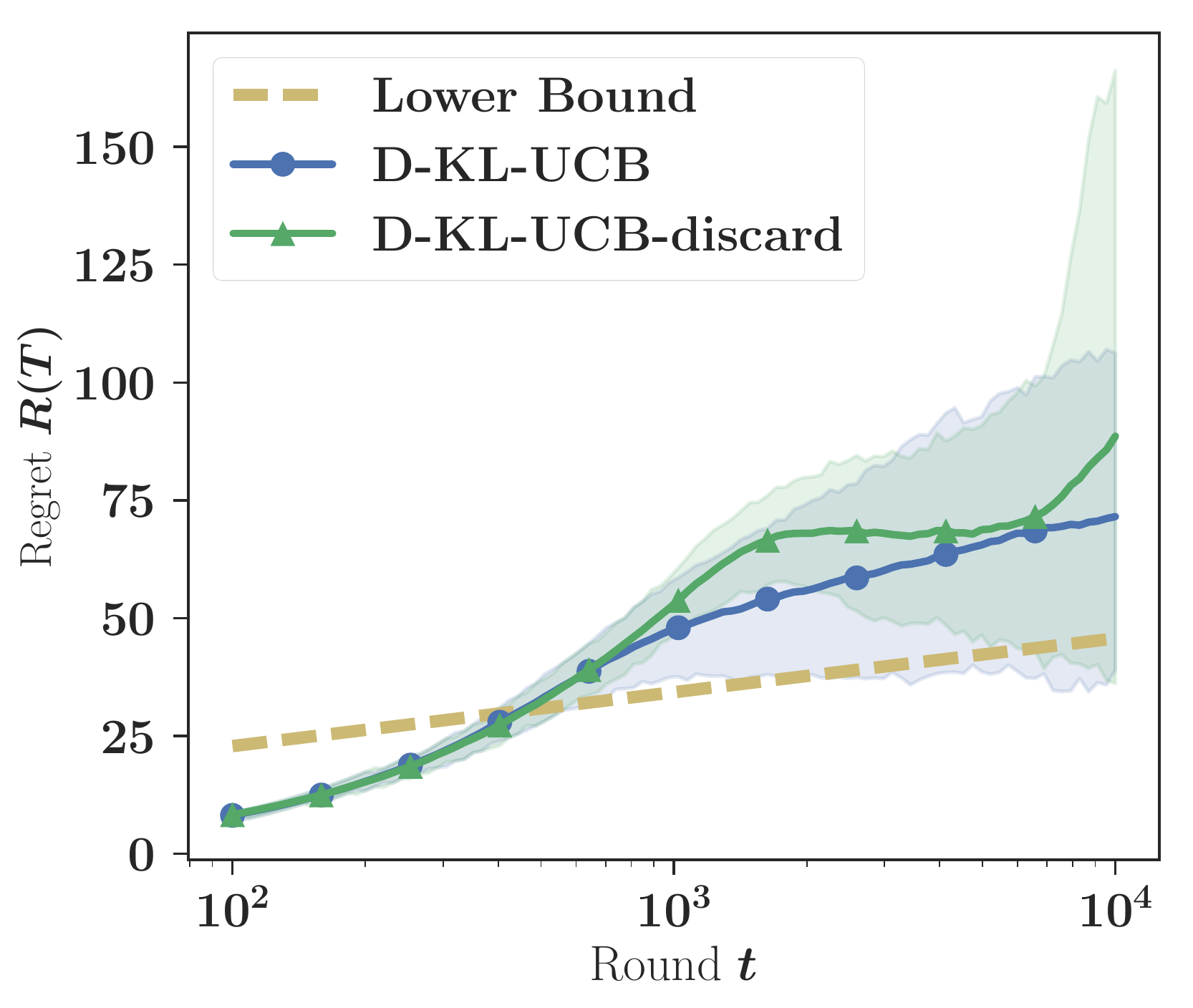}
\includegraphics[width=0.49\linewidth]{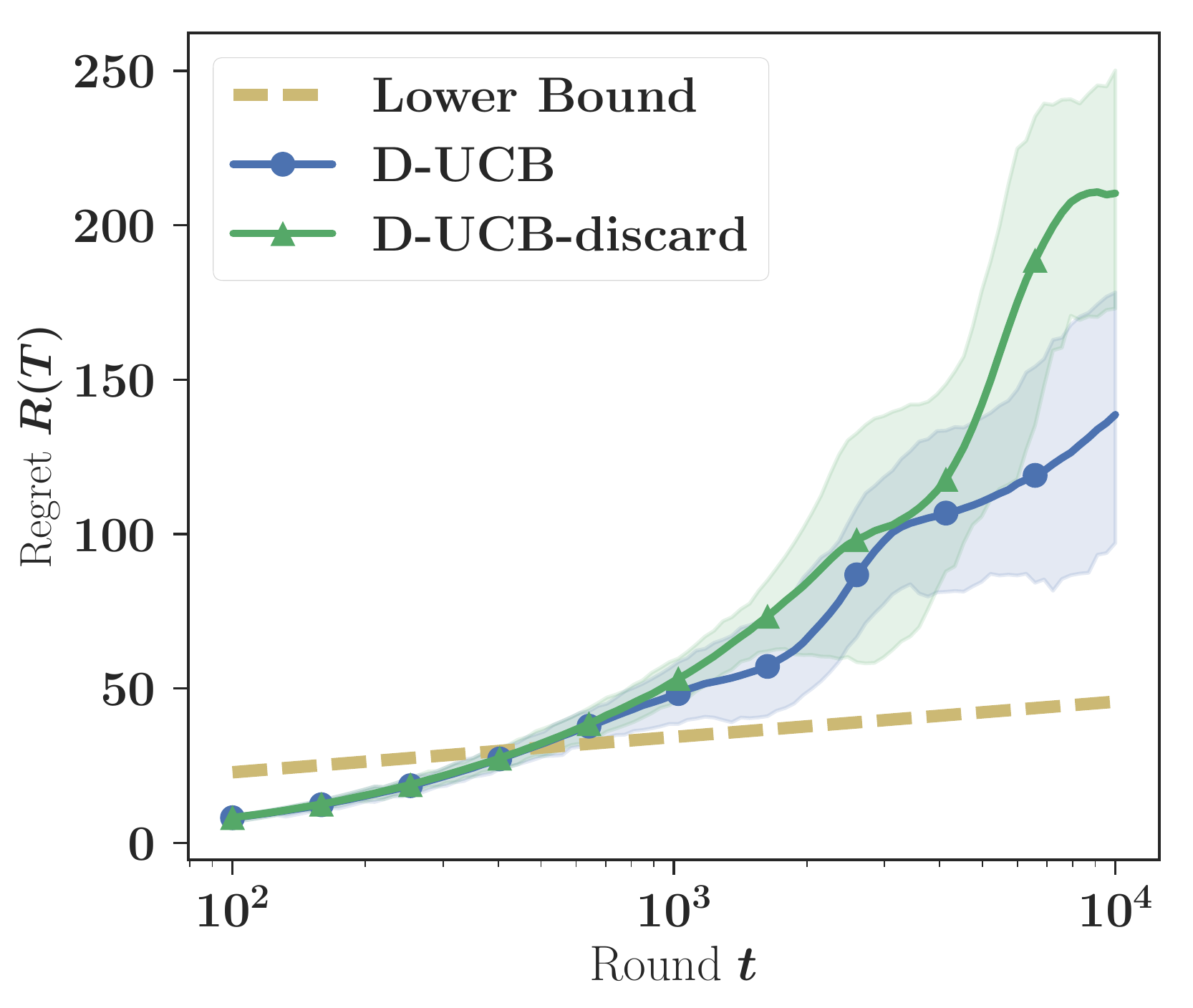}
\caption{Expected regret of \textsc{d-ucb} and \textsc{d-klucb} in the censored setting vs  the equivalent discarding policies when $\mu=500$ and $m=1000$. Results are averaged over 200 independent runs.}
\label{fig:disc}
\end{figure}

\paragraph{\delayedUCB\ and \delayedKLUCB\ vs. \textsc{Discarding}.}

In this section, we illustrate the good empirical initial performance of \delayedUCB\ and \delayedKLUCB, when compared to the heuristic \textsc{Discarding} approach presented in Section~\ref{sec:algorithms}.

Figure~\ref{fig:disc}  compares results for both \delayedUCB\  and \delayedKLUCB\  with $\theta =(0.1,0.05, 0.03)$, $T=10000$, $\mu=500$ and $m=1000$ in the censored setting. 
We observe that discarding policies incur a linear regret phase at the beginning of the learning and happen to catch up with the expected regret growth rate only after a large number of rounds. These figures reveal a non-negligible gap in performance between the naive \textsc{Discarding} approach and our delay-adapted quasi-optimal algorithms.

\section{CONCLUSION}

The stochastic delayed bandit setting introduced in this work addresses an
important problem in many applications where the feedback associated to each
action is delayed and censored, due to the ambiguity between conversions that
will never happen and conversions that will occur at some later -- perhaps
unobservable-- time. Under the hypothesis that the distribution of the delay is
known, we provided a complete analysis of this model as well as simple and
efficient algorithms. An interesting generalization of the present work would
be to relax the model hypothesis and estimate the delay distribution on-the-go,
possibly using context-dependent delay distributions.


\bibliography{references}
\bibliographystyle{plain}

\newpage

\appendix
\onecolumn

\section{Concentration results}
\label{ap:concentration}

\subsection{Poissonization of the KL indices}
\label{ap:Poisson}
We require variants of Lemma 9 and Theorem 10 in \cite{garivier2011klucb} adapted to our setting.
\begin{lemma}
\label{lemma:laplace}
For $\theta \in [0,1]$, $(\tau_i)_{1 \leq i \leq L} \in [0,1]$, let
$(X_{i,j})_{1\leq i \leq L, j \geq 1}$ be a collection of independent Bernoulli
random variables such that $ \mathds{E}(X_{i,j}) = \tau_i\theta$ and
$(\epsilon_{i,j}) \in \{0,1\}$ associated deterministic indicators. For
$1\leq i \leq L$, denote by $n_i = \sum_{j=1}^\infty \epsilon_{i,j}$ and whe
shall assume that all $n_i$ are finite an that at least one of them is
non-zero. Let $X = \sum_{i=1}^L \sum_{j=1}^\infty \epsilon_{i,j} X_{i,j}$
and denote by $\phi(\lambda) = \log \mathds{E}\left[\exp(\lambda X)\right]$ its
log-Laplace transform and by
$\phi^*(x) = \sup_{\lambda} x\lambda - \phi(\lambda)$ the associated convex
conjugate (Fenchel--Legendre transform).

Then, for all $\lambda \in \mathds{R}$,
\begin{equation}
\label{eqap:loglaplace-bound:pois}
\phi(\lambda) \leq \left(\sum_{i=1}^L \tau_i n_i \right) \theta \left(e^\lambda - 1\right) ,
\end{equation}
and, for all $x \geq 0$,
\begin{equation}
\label{eqap:divergence-bound:pois}
\phi^*(x) \geq \left(\sum_{i=1}^L \tau_i n_i \right) \dP\left(\frac{x}{\sum_{i=1}^L \tau_i n_i},\theta\right) ,
\end{equation}
where $\dP(p,q) = p\log p/q + q-p$ denotes the Poisson Kullback-Leibler divergence. 
\end{lemma}

\begin{proof}
By direct calculation,
\[
\phi(\lambda) = \sum_{i=1}^L n_i \log\left(1 - \tau_i\theta + \tau_i\theta e^{\lambda}\right) .
\]

The function $\tau_i \to \log\left(1 - \tau_i\theta + \tau_i\theta e^{\lambda}\right)$ is a strictly concave function on $[0,1]$ and we upper bound it by its tangent in 0, that is,
\begin{equation*}
\log\left(1 - \tau_i\theta + \tau_i\theta e^{\lambda}\right)  \leq \tau_i \theta (e^{\lambda}-1) ,
\end{equation*}
which yields~\eqref{eqap:loglaplace-bound:pois} upon summing on $i$.

The r.h.s. of~\eqref{eqap:loglaplace-bound:pois} is easilly recognized as the log-Laplace transform of the Poisson distribution with expectation $\left(\sum_{i=1}^L \tau_i n_i \right) \theta$. To obtain~\eqref{eqap:divergence-bound:pois}, we use the observations that $x \lambda - a \left(e^\lambda - 1\right)$ is maximized for $\lambda = \log (x/a)$ where it is equal to $\dP(x,a)$ as well as the fact that $\dP(\tau x,\tau a) = \tau \dP(x, a)$.
\end{proof}

Lemma~\ref{lemma:laplace} bounds the log-Laplace transform of the Bernoulli distribution with that of the Poisson distribution with the same mean and uses the stability of the Poisson distribution. Using $\dP(p, q)$ instead of $d(p,q)$ --where $d(p,q) = p\log p/q + (1-p)\log[1-p)/(1-q)$ denotes the Bernoulli Kullback-Leibler divergence-- will of course induce a performance gap, which is however not significant for low values of the probabilities as shown by the Lemma~\ref{lem:div_bound}, which we recall below.

\newcounter{temp}
\setcounter{temp}{\thetheorem}
\setcounter{theorem}{7}
\begin{lemma}\label{8}
	For $0<p<q<1$,
$$
  (1-q) d(p,q) \leq \dP(p,q) \leq d(p,q).
$$
\end{lemma}
\setcounter{theorem}{\thetemp}

\begin{proof} For the upper bound,
  \begin{equation*}
    d(p,q) - \dP(p,q) = (1-p)\log\frac{1-p}{1-q} - (q - p)
    = -(1-p) \log(1+\frac{p-q}{1-p}) + (p-q) \geq 0,
  \end{equation*}
  using $-\log(1+x) \geq -x$.

For the lower bound,
\begin{equation}
\label{q:sfdjqos}
\dP(p,q) - (1-q)d(p,q) = qp \log \frac{p}{q} + q -p - (1-q)(1-p)\log\left(\frac{1-p}{1-q}\right) .
\end{equation}
One has $\dP(q,q) = d(q,q) = 0$ and the derivative of~\eqref{q:sfdjqos} wrt. $p$ is equal to
\[
  q \log \frac{p}{q} + (1-q) \log\left(\frac{1-p}{1-q}\right) = - d(q,p) \leq 0.
\]
Hence, $\dP(p,q) - (1-q)d(p,q)$ is positive when $p \leq q$.
\end{proof}

We can now prove the concentration result stated in Lemma~\ref{lemma:KL-concentration} that we recall below for readability purpose.

\setcounter{temp}{\thetheorem}
\setcounter{theorem}{6}
\begin{lemma}
	Assume that the sequence of pulls is fixed beforehand and let $k$ be an arm in $\{1,...,K\}$. Then for any $\delta > 0$ and for all $t>0$, 
	\[
	\mathds{P}\left( \left\{\hat{\theta}_k(t) < \theta_k \right\} \cap \left\{\tilde{N}_k(t)\dP(\hat{\theta}_k(t),\theta_k) > \delta\right\} \right) <  
	e^{-\delta}.
	\]
	where $\dP(p,q) = p\log p/q + q-p$ denotes the Poisson Kullback-Leibler divergence.	
\end{lemma}
\setcounter{theorem}{\thetemp}

\begin{proof}
To bound $\mathds{P}(\hat{\theta}_k(t) < x) = \mathds{P}(S_k(t) < \tilde{N}_k(t)x)$, for $0<x<\theta_k$, apply Chernoff's method using the result of Lemma~\ref{lemma:laplace} to obtain
	\[
	\mathds{P}(\hat{\theta}_k(t) < x) \leq e^{-\tilde{N}_k(t) \dP(x, \theta_k)}.
	\]
	Using that $x \mapsto \dP(x,\theta_k)$ is decreasing on $[0,\theta_k]$, we can apply it on both side of the inequality on the left-hand side to obtain
	\[
		\mathds{P}\left(\left\{\hat{\theta}_k(t) < \theta_k \right\} \cap \left\{\tilde{N}_k(t)\dP(\hat{\theta}_k(t),\theta_k)>\tilde{N}_k(t)\dP(x,\theta_k)\right\}\right) \leq e^{-\tilde{N}_k(t) \dP(x, \theta_k)}.
	\]
	Denoting $\delta = \tilde{N}_k(t) \dP(x, \theta_k)$ yields the desired result.
\end{proof}

\begin{theorem}
	\label{th:self-normalized}
	Consider $(\tau_i)_{1 \leq i\leq L} \in [0,1]$, $\theta \in(0, 1)$, and independent sequences $(X_i(s))_{s\geq1}$ of independent Bernoulli random variables such that $\mathds{E} X_i(s) = \tau_i\theta$. Let $\mathcal{F}_t$ denote an increasing sequence of sigma-fields, such that for each $t$ and all $i$, $\sigma(X_i(1),\ldots,X_i(t))\subset \mathcal{F}_t$. Also consider a predictable sequence of indicator variables $\epsilon_i(s) \in \{0,1\}$, that is, such that $\sigma(\epsilon_1(t+1),\ldots,\epsilon_L(t+1))\subset \mathcal{F}_t$.

       Define
	\[
	S_i(t)=\sum_{s=1}^{t} \epsilon_i(s) X_i(s), \qquad N_i(t) = \sum_{s=1}^t \epsilon_i(s);
	\]
	and the pooled quantities
	\[
	S(t) = \sum_{i=1}^{L} S_i(t), \qquad N(t) = \sum_{i=1}^{L} N_i(t), \qquad \tilde{N}(t) = \sum_{i=1}^{L}\tau_i N_i(t), \qquad \hat{\theta}(t)=\frac{S(t)}{\tilde{N}(t)} .
	\]
	The KLUCB index, defined as,
	\[
	U^{\textsc{kl}}(n) =  \max \left\{ q \in \left[\hat{\theta}(n), \theta_M\right] \, : \, \tilde{N}(n)\dP(\hat{\theta}(n),q)\leq \delta \right\}.
	\]
	satisfies
	\[
	\mathds{P}\left( U(n) \leq \theta\right) \leq e \lceil \delta \log (n) \rceil e^{-\delta} .
	\]
\end{theorem}

\begin{proof}
	The proof is analogous to that of Theorem 10 of \cite{garivier2011klucb} and we only detail the step that differs, namely, the identification of the supermartingale $W_t^\lambda$.

Define, $W_0^\lambda = 1$ and, for $t\geq 1$, 
\[
 W^{\lambda}_t = \exp\left(\lambda S(t) - \tilde{N}(t) \theta\left(e^\lambda-1\right) \right).
\]
\begin{align*}
  \mathds{E}\left[\exp(\lambda(S(t+1)-S(t))) \left\vert \mathcal{F}_t \right. \right] & = \mathds{E}\left[ \left. \exp\left(\lambda \sum_{i=1}^L \epsilon_i(t+1) X_i(t+1)\right)   \right\vert \mathcal{F}_t \right] \\
    & \leq \exp\left( \left(\sum_{i=1}^L \tau_i \epsilon_i(t+1) \right) \theta \left(e^\lambda-1\right) \right) \\
    & = \exp\left( \left(\tilde{N}(t+1)-\tilde{N}(t) \right) \theta \left(\theta e^\lambda-1\right) \right) ,
	\end{align*}
where we have used~\eqref{eqap:loglaplace-bound:pois} and the definition of $\tilde{N}(t)$. Multiplying both sides of the inequality by $\exp\left(\lambda S(t) - \tilde{N}(t+1) \theta \left(e^\lambda-1\right) \right)$ show that $\mathds{E} W_{t+1}^\lambda \leq \mathds{E} W_{t}^\lambda$ and hence that $W_t^\lambda$ is a supermartingale.

The rest of the proof is as in \cite{garivier2011klucb} replacing $N(t)$ by $\tilde{N}(t)$ and $\phi_\mu(\lambda)$ by $\theta\left(e^\lambda-1\right)$.
\end{proof}


\section{Details on the Lower Bound Results}
\label{ap:lower_bound}

We provide here the details of the proof of Theorems~\ref{th:lb-uncens}~
and~\ref{th:lb-cens}. The key result that we use is a lower bound on the
log-likelihood ratio under two alternative bandit models $\theta$ and $\theta'$ that
do not have the same best arm. Namely, according to Lemma~1 of
\cite{kaufmann2015complexity}, we have 
\[\liminf_{T\to \infty}
\frac{\mathds{E}[\ell_T]}{\log(T)} \geq 1.
\]
Now, considering specific changes of measures $\theta'$ that only modify the distribution of one single suboptimal arm, 
we are going to obtain lower bounds on each expected number of pulls $\mathds{E}[N_k(T)]$ for $k\neq 1$ as in Appendix~B of \cite{kaufmann2015complexity}.

\paragraph{Uncensored Setting: } As argued in Section~\ref{sec:lower-bound}, in the uncensored setting the likelihood of the observations is	
\[
	\mathds{E}_\theta\left[ \ell_T\right] = \sum_{s=1}^T d(\theta_{A_s}\tau_{T-s}, \theta_{A_s}'\tau_{T-s}) .
\]
Now, fix arm $k \neq 1$ and for $\epsilon >0$, consider 
$\theta'=(\theta_1,\ldots,\theta_{k-1},\theta_1+\epsilon,\ldots,\theta_K)$. For this change of measure, the expected log-likelihood only contains the terms involving arm $k$:
\[
	\mathds{E}_\theta\left[ \ell_T\right] = \sum_{s=1}^T \mathds{1}\{A_s=k\} d(\theta_k\tau_{T-s}, (\theta_1+\epsilon)\tau_{T-s}) .
\]
Now, in order to obtain an expression that involves $\mathds{E}[N_k(T)]$, we need to bound from above this sum using Lemma~5 of the Appendix~B of \cite{katariya2016stochastic}, which we recall here for completeness.

\begin{lemma}\label{lem:incrkappa}
	Let $p,q$ be any fixed real numbers in $(0,1)$. The function 
	$f:\alpha 
	\mapsto d(\alpha p,\alpha q)$  is convex and increasing on $(0,1)$. As 
	a consequence, for any $\alpha<1$, $d(\alpha p, \alpha q) < d(p,q)$.
\end{lemma}

Thus, for each $s\geq 1$ we have $\tau_{T-s}\leq 1$ and according to the above result,
\[
d(\theta_k, (\theta_1+\epsilon)) \geq d(\theta_k\tau_{T-s}, (\theta_1+\epsilon)\tau_{T-s})
\]
and 
\[
\mathds{E}[N_k(T)] d(\theta_k,\theta_1+\epsilon) \geq \sum_{s=1}^T \mathds{1}\{A_s=k\} d(\theta_k\tau_{T-s}, (\theta_1+\epsilon)\tau_{T-s}).
\]
We obtain
\[
 \liminf_{T\to \infty} \frac{\mathds{E}[N_k(T)] d(\theta_k,\theta_1+\epsilon)}{\log(T)}  \geq \liminf_{T\to \infty} \frac{\mathds{E}[\ell_T]}{\log(T)} \geq 1.
\]
Letting $\epsilon \to 0$ yields
\[
 \liminf_{T\to \infty} \frac{\mathds{E}[N_k(T)] }{\log(T)} \geq \frac{1}{d(\theta_k,\theta_1)}.
\]
In order to bound the expected regret $L_T$, we use the inequality~\eqref{eq:regret:relationship} from Lemma~\ref{lemma:regret}:
\[
L_T\geq \sum_{k=2}^K (\theta_{1} - \theta_k) \left(\mathds{E}[N_k(t)] - \mu\right),
\]

where $\mu = \mathds{E}[D_s]$. We now lower bound each $\mathds{E}[N_k(t)]$ and under the assumption that $\mathds{E}[D_s]<\infty$ and we use that $\liminf_{T\to \infty} \mu/\log(T) = 0$ to obtain
\[
 \liminf_{T\to \infty} \frac{L_T }{\log(T)} \geq  \liminf_{T\to \infty} \frac{\sum_{k=2}^K (\theta_{1}- \theta_k) \left(\mathds{E}[N_k(t)] - \mu\right)}{\log(T)} \geq \sum_{k=2}^{K}\frac{(\theta_{1}- \theta_k)}{d(\theta_k,\theta_1)}.
\]

\paragraph{Censored Setting: } 
The proof in the Censored Setting follows the same step as the proof above expect for the fact that we do not require Lemma~5 of \cite{katariya2016stochastic} in order to bound the log-likelihood ratio. We directly have 
\begin{align*}
	\mathds{E}_\theta\left[ \ell_T\right] =& \sum_{s=1}^{T-m} d(\theta_{A_s}\tau_{m}, \theta_{A_s}'\tau_{m})+ \sum_{s=T-m}^T d(\theta_{A_s}\tau_{T-s}, \theta_{A_s}'\tau_{T-s}).	
\end{align*}
Proceeding as above and considering the adequate change of measure involving only one suboptimal arm $k$ and taking $\epsilon \to 0$, we obtain 
\[
 \liminf_{T\to \infty} \frac{\mathds{E}[N_k(T)] d(\tau_m\theta_k,\tau_m\theta_1) + \sum_{s=T-m}^T d(\theta_k\tau_{T-s}, \theta_1\tau_{T-s})}{\log(T)}  = 
 \liminf_{T\to \infty} \frac{\mathds{E}[N_k(T)] d(\tau_m\theta_k,\tau_m\theta_1)}{\log(T)}  \geq 1 ,
\]
where we used the fact that the second term of the sum in the left-hand side is finite.
The end of the proof is similar to the uncensored setting case treated above where we can simply bound the regret according to Eq.~\eqref{eq:regret_th} as
\[
L_T\geq \sum_{k=2}^k \tau_m(\theta_{1}- \theta_k) \mathds{E}[N_k(T-m)] + \sum_{s=T-m+1}^{T} \tau_{T-s}(\theta_1 - \theta_{A_s}) 
\]
in order to obtain the asymptotic lower bound.


\section{Analysis of \delayedUCB\ and \delayedKLUCB}

In order to control the empirical averages of the rewards of each arm for different values of $N_k(t)$, we introduce the notation $\hat{\theta}_{k,s}:=\sum_{u=1}^s X_{k,u}/s$ for the mean over the first $s$ pulls of $k$.

\subsection{\delayedUCB}
\label{ap:ucb}

In this section, we provide the complete proof of Theorem~\ref{th:ucb}. 





	We decompose the regret after bounding by $1$ the first $m$ losses of the policy :
	\[
	L_{\textsc{ucb}}(T) \leq m + \sum_{k>1} \tau_m \Delta_k \mathds{E}\left[ \sum_{t>m}^{T}\mathds{1}\{A_t = k\} \right].
	\]
	Hence we only need to bound the number of suboptimal pulls, as in the 
	seminal proof by \cite{auer2002finite}.
	For any suboptimal $k>1$, we have:
	\begin{align*}
		\mathds{E}[N_k(T)]  \leq 1 & + 
		\sum_{t=K+1}^{T}\mathds{P}\left( U^{\textsc{UCB}}_1(t) < \theta_1 \right) \\
		& + \sum_{t=K+1}^{T} \mathds{P}\left( A_{t+1}=k, U^{\textsc{UCB}}_k(t) \geq \theta_1 \right).
	\end{align*}

	While the first term is simply handled by Proposition~\ref{prop:controlUCB} and is $O(1/\epsilon^3)=o(\log(T)$, 
	the second one must be controlled as in the original proof of UCB1 by \cite{auer2002finite} using the fact that for all $t>m$,
	\[
	\frac{N_k(t)}{\tilde{N}_k(t)} \leq \frac{N_k(t-m)+m}{\tilde{N}_k(t)} \leq \frac{1}{\tau_m} + \frac{m}{\tau_m N_k(t-m)}
	\]
	that allows us to upper-bound the optimistic indices as 
	\[
	\hat{\theta}_k(t) + \left( \frac{1}{\tau_m} + \frac{m}{\tau_m N_k(t-m)} \right) \sqrt{\frac{\beta_\epsilon(t)}{2 N_k(t)}} 
	\geq U^{\text{UCB}}_k(t).
	\] 
	Then, we use this upper bound on the indices in order to bound the relevant sum of probabilities.
	\begin{align*}
		&\sum_{t=m+1}^{T} \mathds{P}\left(A_{t+1}=k, U^{\text{UCB}}_k(t) \geq \theta_1 \right) \\
		& \leq \mathds{E}\left[ \sum_{s\geq 1 } \mathds{1} \left\{ \hat{\theta}_{i,s} 
		+ \left( \frac{1}{\tau_m} + \frac{m}{\tau_m s} \right)
		\sqrt{\frac{\beta_\epsilon(t)}{2s}} \geq \theta_i + \Delta_i \right\} \right] .
	\end{align*}

In order to upper-bound this expectation, we first introduce the quantity $\underline{s}_i >0$  defined by 
$$
\left( \frac{1}{\tau_m} + \frac{m}{\tau_m \underline{s}_i} \right)
		\sqrt{\frac{\beta_\epsilon(t)}{2\underline{s}_i}} =  \Delta_i,
$$
that we rewrite, with the introduction of $\gamma_i>0$, as 
$$
\underline{s}_i= \frac{\beta_\epsilon(t)}{2\tau_m^2\Delta_i^2} (1+\gamma_i)^2 \quad \text{ so that we get } \quad \left( 1+ \frac{m}{ \underline{s}_i} \right)\frac{1}{1+\gamma_i} = 1.
$$
Simple computations finally leads to, if $\gamma_i \leq 1$, 
\begin{align*}
\frac{2m\tau_m^2\Delta_i^2}{\beta_\epsilon(t)} = \gamma_i(1+\gamma_i)^2 \leq 4 \gamma_i.
\end{align*}
As a consequence, if $T$ is big enough (so that the left hand side is smaller than 4), we get that 
$$
\underline{s}_i \leq \frac{(1+\varepsilon)\log(T)}{2\tau_m^2\Delta^2_i}(1+\gamma_i)^2 \leq \frac{(1+\varepsilon)\log(T)}{2\tau_m^2\Delta^2_i}(1+3\gamma_i) \leq  \frac{(1+\varepsilon)\log(T)}{2\tau_m^2\Delta^2_i} + m.
$$
We now focus on the sum to upper-bound:
\begin{align*}
\mathds{E}\left[ \sum_{s\geq 1 } \mathds{1} \left\{ \hat{\theta}_{i,s} 
		+ \left( \frac{1}{\tau_m} + \frac{m}{\tau_m s} \right)
		\sqrt{\frac{\beta_\epsilon(t)}{2s}} \geq \theta_i + \Delta_i \right\} \right]  &\leq \lceil\underline{s}_i\rceil+1 + \sum_{s > \lceil\underline{s}_i\rceil}e^{ -2s \left(\Delta_i - \left( \frac{1}{\tau_m} + \frac{m}{\tau_m s} \right) \sqrt{\frac{\beta_\epsilon(t)}{2s}} \right)^2}\\
		& \leq  \underline{s}_i+2 + \sum_{s > \lceil\underline{s}_i\rceil}e^{ -2 \left(\sqrt{s}\Delta_i - \left( 1+ \frac{m}{ \underline{s}_i} \right) \sqrt{\frac{\beta_\epsilon(t)}{2\tau_m^2}} \right)^2},
\end{align*}
where we used the Chernoff's inequality for bounded random variables.


Standard computations (comparisons between sums and integrals) give the following
\begin{align*}
\sum_{s > \lceil\underline{s}_i\rceil}e^{ -2 \left(\sqrt{s}\Delta_i - \left( 1+ \frac{m}{ \underline{s}_i} \right) \sqrt{\frac{\beta_\epsilon(t)}{2\tau_m^2}} \right)^2}
& \leq  \int_{\underline{s}_i}^\infty e^{ -2 \left(\sqrt{s}\Delta_i - \left( 1+ \frac{m}{ \underline{s}_i} \right) \sqrt{\frac{\beta_\epsilon(t)}{2\tau_m^2}}\right)^2} ds \\
&\leq \frac{1}{2\Delta_i^2}\left(1+\frac{\sqrt{2\pi}}{4}\left( 1+ \frac{m}{ \underline{s}_i} \right) \sqrt{\frac{\beta_\epsilon(t)}{2\tau_m^2}}\right)\\
& \leq  \frac{1}{2\Delta_i^2}\left(1 + \frac{\sqrt{2\pi}}{4}\sqrt{\underline{s}_i}\Delta_i\right) \\
& \leq\frac{1}{2\Delta_i^2}\left(1 + \frac{\sqrt{\pi}}{4}\sqrt{\frac{(1+\epsilon)\log(T)}{\tau_m^2}} +  \frac{\sqrt{\pi}}{4}\sqrt{m}\right).
\end{align*}
As a consequence, we have just proved that 
$$
\sum_{t=m+1}^{T} \mathds{P}\left(A_{t+1}=k, U^{\text{UCB}}_k(t) \geq \theta_1 \right) \leq \frac{(1+\epsilon)\log(T)}{2\tau_m^2\Delta^2_i} + o(\log(T))\ .
$$
More precisely, combining all our claims yields that
$$
L_{\textsc{ucb}}(T) \leq  \frac{(1+\epsilon)\log(T)}{2\tau_m^2\Delta_i}  + O\left( \frac{1}{\Delta_i}\sqrt{\frac{(1+\epsilon)\log(T)}{2\tau_m^2}}\right) + O\left(\frac{1}{\Delta_i}\frac{1}{\epsilon^3}\right) + O\left(\frac{\sqrt{m}}{\Delta_i}+m\right),
$$
and the result follows.

%
%

\subsection{\delayedKLUCB}
\label{ap:klucb}


We follow the steps of \cite{garivier2011klucb} and decompose the regret as
	\begin{align*}
		\mathds{E}[N_k(T)]  \leq & 1 + m-K
		\sum_{t=m+1}^{T}\mathds{P}\left( U^{\textsc{kl}}_1(t) < \theta_1 \right)  + \sum_{t=m+1}^{T} \mathds{P}\left( A_{t+1}=k, U^{\textsc{KL}}_k(t) \geq \theta_1 \right)\\
		&\leq 1 + m -K + \sum_{t=m+1}^{T}\mathds{P}\left( U^{\textsc{KL}}_1(t) < \theta_1 \right) +  \sum_{t=m+1}^{T} \mathds{P}\left( A_{t+1}=k, U^{\textsc{kl}}_k(t) \geq \theta_1 \right).
	\end{align*}

	The first term of the above sum is handled by Theorem~\ref{th:self-normalized} that shows that it is $o(\log(T))$. We must now bound the second sum corresponding to the cases when suboptimal indices reach the optimal mean $\theta_1$. 
	To proceed, we simply notice that for all $t$, $\tilde{N}_k(t) \geq \tau_m N_k(t-m)$. We define an alternative optimistic index that upper bounds $U^{\textsc{kl}}_k(t)$ for $t>m$: 
	\[
	U^{\textsc{kl}}_k(t) \leq \argmax_{q \in [\hat{\theta}_k,1]} \{q | \tau_m N_k(t-m) \dP(\hat{\theta}_k(t),q)\leq \beta_\epsilon(t)\}
	:=U^{\textsc{kl}+}_k(t).
	\]

	Now we can finish the proof following the steps of the proof of Theorem 2 in \cite{garivier2011klucb}.
	First, we denote 
	\[
	K_k(T) = \frac{(1+\eta)\beta_\epsilon(t)}{ d(\tau_m\theta_k,\tau_m\theta_1)}
	\]
	and we decompose the second sum after bounding the first $K_k(T)$ terms by 1 and bounding the remaining terms in a similar way as in Lemma 11 of \cite{joulani2013online}:
	\begin{align*}
		\sum_{t=m+1}^{T} &\mathds{P}\left( A_{t+1}=k, U^{\textsc{kl}+}_k(t) \geq \theta_1 \right) 
		\leq K_k(T) 
		+ \sum_{t \geq K_k(T)+m+1} \mathds{P}\left( A_{t+1}=k, U^{\textsc{kl}+}_k(t) \geq \theta_1 \right)\\
		& \leq K_k(T) 
		+ \mathds{E}\left[ \sum_{t \geq K_k(T)+m+1} \sum_{s=1}^{t} 
		\mathds{1}\left\{ A_{t+1}=k, N_k(t-m)=s, \tau_m s \dP(\hat{\theta}_{k,s},\theta_1)
		\leq \beta_\epsilon(t) \right\} \right]\\
		& \leq K_k(T) 
		+ \mathds{E}\left[ \sum_{s=K_k(T)}^{T} 
		\mathds{1}\left\{ \tau_m s \dP(\hat{\theta}_{k,s},\theta_1)
		\leq \beta_\epsilon(t) \right\}
		\sum_{t = s} ^{T} \mathds{1}\left\{ A_{t+1}=k, N_k(t-m)=s\right\} \right]\\
		&\leq K_k(T) + m \sum_{s\geq K_k(T)} \mathds{P} \left(\tau_m s \dP(\hat{\theta}_{k,s},\theta_1)
		\leq \beta_\epsilon(t) \right)\\
		&\leq K_k(T) + \frac{m C_2(\eta)}{T^{f(\eta)}},
	\end{align*}
	where the last inequality comes from the fact that for all $s\in \{1,\ldots, T\}$, $\sum_{t = 1} ^{T} \mathds{1}\left\{ A_{t+1}=k, N_k(t-m)=s\right\} \leq m$ and from the proof of Fact 2 for exponential family bandits in \cite{cappe2013kullback} that proves the existence of the constants $C_2(\eta)$ and $f(\eta)$ that achieve the bound.
	
	We can now upper bound the regret thanks to the decomposition provided by equation~\ref{eq:regret_th}:
		\[
		L_{\textsc{klucb}}(T) \leq m + \sum_{k>1} \tau_m \Delta_k \mathds{E}\left[ N_k(T)\right] 
		\leq (1+\eta) \beta_\epsilon(t)\sum_{k=2}^{K}\frac{\tau_m \Delta_k}{\dP(\tau_m\theta_k,\tau_m\theta_1)} + o(\log(T)).
		\]

	To obtain the final result, we use Lemma~\ref{lem:div_bound} that shows that for $\theta_k<\theta_1$, $\dP(\tau_m\theta_k,\tau_m\theta_1) > (1-\tau_m \theta_1) d(\tau_m \theta_k, \tau_m \theta_1)$. Thus, 
		\[
		L_{\textsc{klucb}}(T) \leq m + (1+\eta) \frac{\beta_\epsilon(t)}{1-\tau_m \theta_1}
		\sum_{k=2}^{K}\frac{\tau_m \Delta_k}{d(\tau_m\theta_k,\tau_m\theta_1)} + o(\log(T)).
		\]

\section{Additionnal experiments  on delay agnostic policies}

As a last additional contribution to this work, we suggest a distribution-agnostic heuristic 
that estimates the CDF parameters $(\tau_d)_{d\geq 0}$ in an online fashion. 
Indeed, as the delay distribution is assumed to be shared between actions, each observed reward provides an information on the delays that can be exploited to estimate the CDF without having to deal with the exploration-exploitation dilemma. 

\paragraph{Uncensored setting. } In the Uncensored setting and under the geometric assumption on the distribution of the delays, the entire CDF can be retrieved using an estimate of the unique parameter $\lambda = 1/\mu$. 
To this aim, we build an estimate the expected delay at round $t$, $\hat{\mu}(t)$, using a stochastic approximation process with decreasing weights $\alpha_t= 1/t^\gamma$ for $1\geq \gamma \geq 0.5$. When an observation $D_t$ arrives we update
\[
\hat{\mu}(t) \gets (1-\alpha_t) \hat{\mu}(t)+ \alpha_t D_t.
\]
Then we use this estimator as a plug-in quantity to compute $O_k(t)$ defined in \eqref{eq:O} for all $k$. 

\paragraph{Censored setting. } In the Censored setting however, no observation comes after the threshold $m$ and this does not allow us to directly estimate the expected delay $\mu$ as the longest observations are censored. 
To circumvent this problem, we choose to estimate biased parameters for $\tau_1,\ldots,\tau_m$. 
Concretely, we initialize counts for the observed delay values $\delta_0 = (0,\ldots,0)\in \mathds{N}^{m+1}$ (delay can be null). Then, after each observation $D_t$, we increment all the counts $\delta_s$ for $s\geq D_t$. 
Then, the biased empirical CDF is obtained by normalizing those counts by the total number of observations received up to time $t$, $n_d(t)$.
We emphasize that the obtained estimators are biased: 
For each $s\in \{0,\ldots,m\}, \mathds{E}[\delta_s(t)/n_d(t)] = \tau_s/\tau_m$ as all observed delays are smaller or equal to $m$.
 Thus, plugging those estimates in $\tilde{N}_k(t)$ actually allows to have an estimate of $\tilde{N}_k(t)/\tau_m$ instead of $\tilde{N}_k(t)$ and consequently an estimate of $\tau_m \theta_k$ instead of $\theta_k$.

\begin{figure}[htb]
\centering
\includegraphics[width=0.4\linewidth]{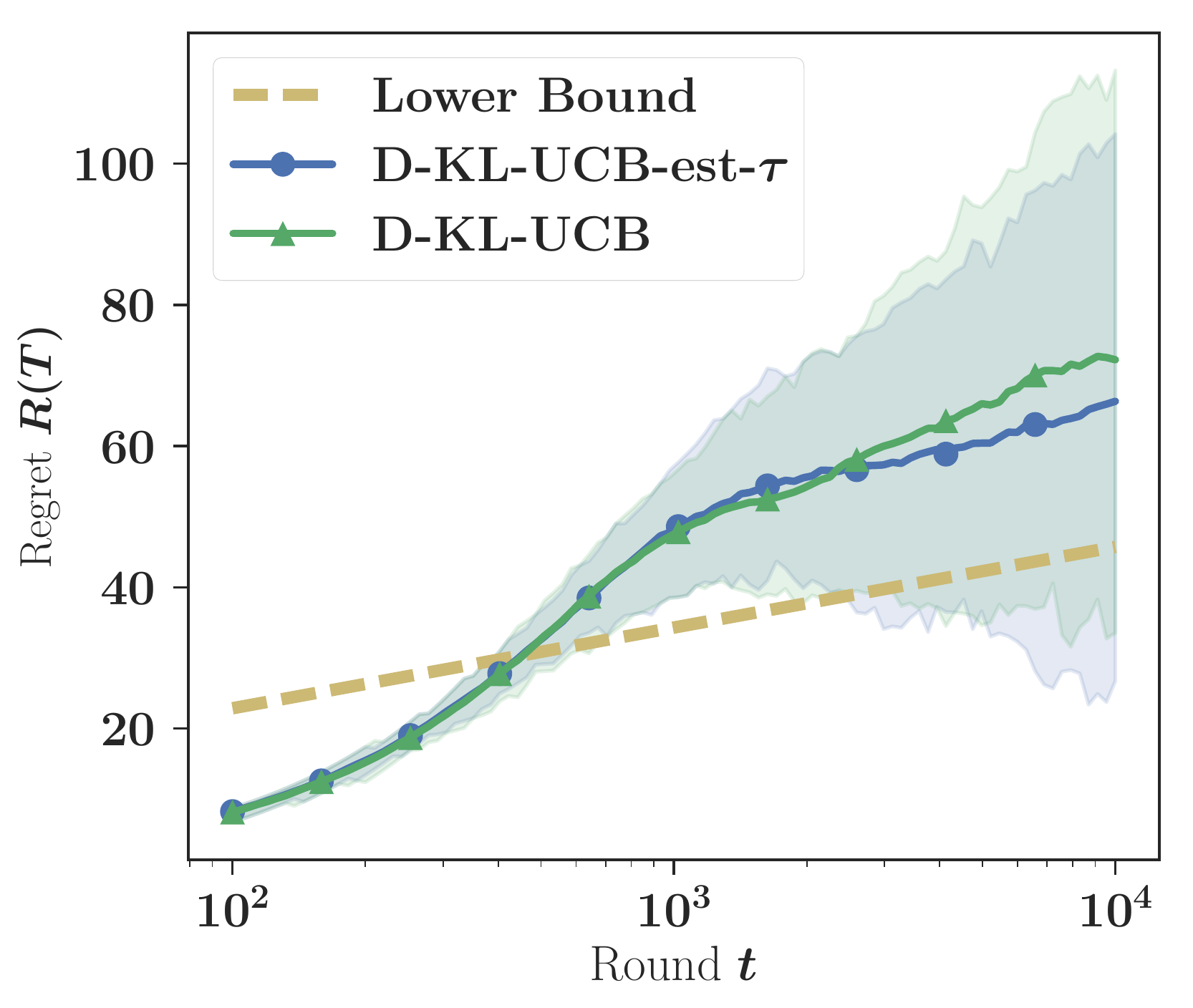}
\includegraphics[width=0.4\linewidth]{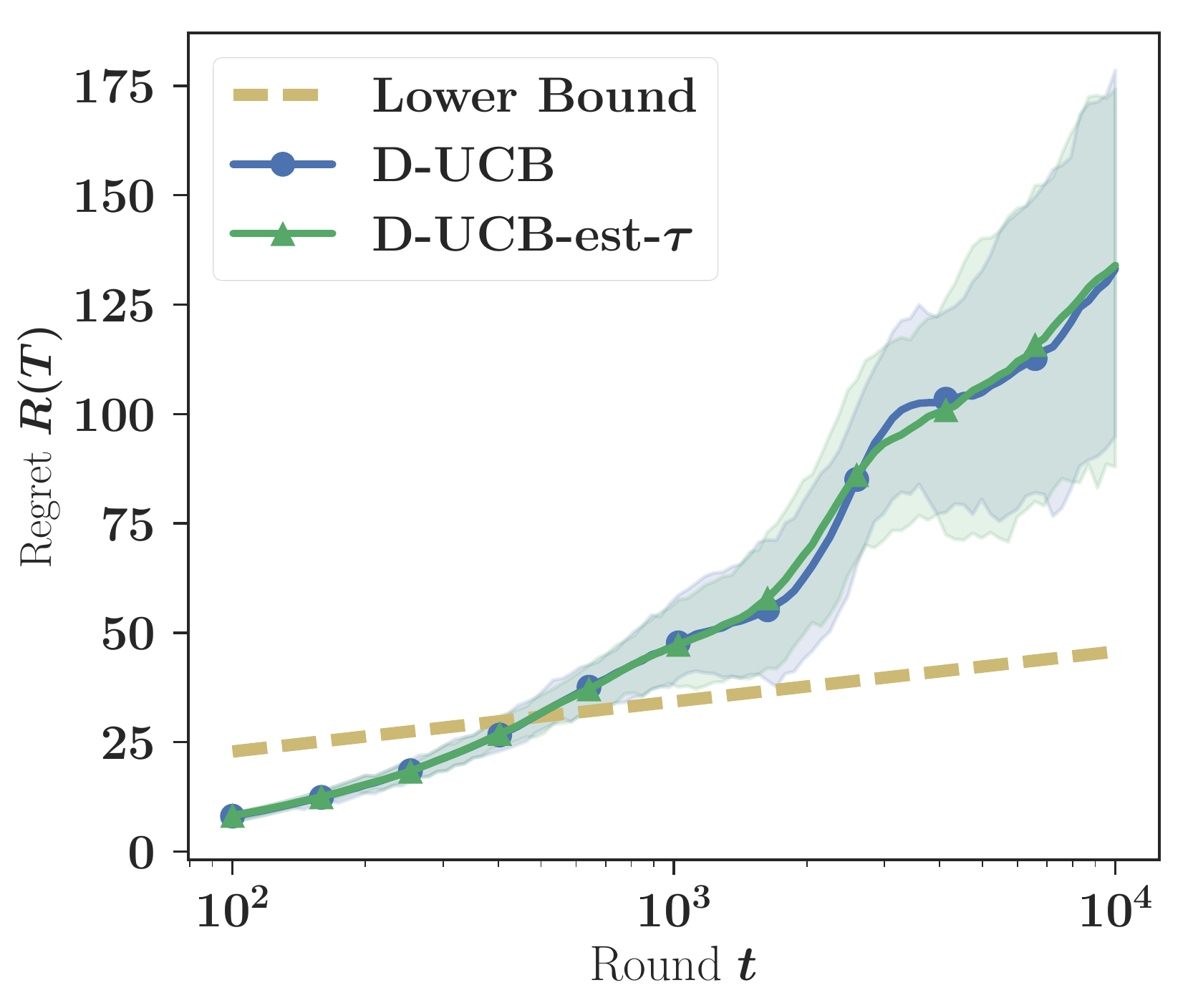}
\caption{Expected regret of \delayedKLUCB~ with and without online estimation of the CDF in both the censored and uncensored setting. Results are averaged over 100 independent runs.}
\label{fig:est}
\end{figure}

Figure~\ref{fig:est} compares both our policies to its equivalent, delay-agnostic heuristic using the same confidence intervals with plug-in estimates of the $(\tau_d)_{d\geq 0}$. It is clear from these experiments that using delay parameters estimated on-the-go does not hurt the cumulated regret overall.


\end{document}